\documentclass[10pt,twocolumn,letterpaper]{article}

\usepackage[pagenumbers]{cvpr} 

\usepackage{graphicx}
\usepackage{amsmath}
\usepackage{amssymb}
\usepackage{booktabs}

\usepackage[pagebackref,breaklinks,colorlinks]{hyperref}

\usepackage[capitalize]{cleveref}
\crefname{section}{Sec.}{Secs.}
\Crefname{section}{Section}{Sections}
\Crefname{table}{Table}{Tables}
\crefname{table}{Tab.}{Tabs.}

\def\cvprPaperID{9093} 
\def\confName{CVPR}
\def\confYear{2022}

\usepackage{tabularx}
\usepackage{floatrow}
\usepackage{multirow}
\usepackage{wrapfig}
\usepackage{amsthm}

\newcommand{\filluptopage}[1]{%
  \clearpage
  \loop\ifnum\value{page}<#1\relax
    \null\clearpage
  \repeat
  \loop\ifnum\value{page}=#1\relax
    \null\clearpage
  \repeat
}

\makeatletter
\def\blfootnote{\xdef\@thefnmark{}\@footnotetext}
\makeatother

\newcommand{\x}{\mathbf{x}}
\newcommand{\g}{\mathbf{g}}

\newcommand{\bphi}{\boldsymbol{\phi}}

\newcommand{\G}{\mathbf{G}}
\newcommand{\Gx}{\mathbf{G}_{\mathbf{x}}}
\newcommand{\Id}{\mathbf{I}}

\newcommand{\zero}{\mathbf{0}}

\newcommand{\Rot}{\mathbf{R}}

\newcommand{\q}{\mathbf{q}}

\newcommand{\bb}{\mathbf{b}} 
\newcommand{\M}{\mathbf{M}}

\renewcommand{\u}{\mathbf{u}}
\renewcommand{\v}{\mathbf{v}}

\newcommand{\U}{\mathbf{U}}

\newcommand{\V}{\mathbf{V}}

\newcommand{\y}{\mathbf{y}}

\newcommand{\Sphere}{\mathcal{S}}
\newcommand{\R}{\mathbb{R}}

\newcommand{\Amb}{\mathcal{X}}
\newcommand{\Man}{\mathcal{M}}

\newcommand{\curve}{\gamma}
\newcommand{\dcurve}{\dot{\gamma}}
\newcommand{\geo}{\curve_\v}
\newcommand{\dgeo}{\dcurve_\v}
\newcommand{\len}{\ell}
\newcommand{\Exp}{\mathrm{Exp}}
\newcommand{\Log}{\mathrm{Log}}
\newcommand{\Loss}{\mathcal{L}}
\newcommand{\Rx}{R_{\x}}
\newcommand{\TiM}{\T_{\Id}\Man}
\newcommand{\TxM}{\T_{\x}\Man}
\newcommand{\T}{\mathcal{T}}
\newcommand{\TM}{\mathcal{T}\mathcal{M}}
\newcommand{\grad}[1]{\mathrm{grad}#1}
\newcommand{\SO}{\mathrm{SO}(3)}

\newcommand{\lie}{\boldsymbol{\omega}}
\newcommand{\gt}{\mathrm{gt}}

\newtheorem{lemma}{Lemma}

\newtheorem{dfn}{Definition}

\begin{document}
\title{Projective Manifold Gradient Layer for Deep Rotation Regression}

\author{
Jiayi Chen\textsuperscript{1,2} \quad 
Yingda Yin\textsuperscript{1} \quad 
Tolga Birdal\textsuperscript{3,4} \quad 
Baoquan Chen\textsuperscript{1} \quad 
Leonidas J.~Guibas\textsuperscript{3} \quad 
He Wang\textsuperscript{1$^\dagger$} \\
\textsuperscript{1}CFCS, Peking University \quad 
\textsuperscript{2}Beijing Institute for General AI \\
\textsuperscript{3}Stanford University \quad
\textsuperscript{4}Imperial College London
}

\maketitle

\begin{abstract}
Regressing rotations on SO(3) manifold using deep neural networks is an important yet unsolved problem. 
The gap between the Euclidean network output space and the non-Euclidean SO(3) manifold imposes a severe challenge for neural network learning in both forward and backward passes.
While several works have proposed different regression-friendly rotation representations, very few works have been devoted to improving the gradient backpropagating in the backward pass. 
In this paper, we propose a manifold-aware gradient that directly backpropagates into deep network weights. 
Leveraging Riemannian optimization to construct a novel projective gradient, our proposed regularized projective manifold gradient (RPMG) method helps networks achieve new state-of-the-art performance in a variety of rotation estimation tasks.
Our proposed gradient layer can also be applied to other smooth manifolds such as the unit sphere. Our project page is at  \href{https://jychen18.github.io/RPMG}{https://jychen18.github.io/RPMG}.

\end{abstract}
\let\thefootnote\relax\footnote{$^\dagger$: He Wang is the corresponding author (hewang@pku.edu.cn).}
\vspace{-2mm}
\section{Introduction}\label{sec:intro}
\vspace{-2mm}
Estimating rotations is a crucial problem in visual perception that has broad applications, \textit{e.g.}, in object pose estimation, robot control, camera relocalization, 3D reconstruction and visual odometry~\cite{kendall2015posenet,bui20206d,wang2019densefusion,gojcic2020learning,dong2020robust}. Recently, with the proliferation of deep neural networks, learning to accurately regress rotations is attracting more and more attention. However, the non-Euclidean characteristics of rotation space make accurately regressing rotation very challenging.

As we know, rotations reside in a non-Euclidean manifold, $\SO$ group, whereas the unconstrained outputs of neural networks usually live in Euclidean spaces.
This gap between the neural network output space and $\SO$ manifold becomes a major challenge for deep rotation regression, thus tackling this gap becomes an important research topic.
One popular research direction is to design learning-friendly rotation representations, \textit{e.g.}, 6D continuous representation from \cite{zhou2019continuity} and 10D symmetric matrix representation from \cite{peretroukhin_so3_2020}. 
Recently, Levinson \textit{et al.}\cite{levinson2020analysis} adopted the vanilla 9D matrix representation discovering that simply replacing the Gram-Schmidt process in the 6D representation~\cite{zhou2019continuity} with symmetric SVD-based orthogonalization can make this representation superior to the others.

Despite the progress on discovering better rotation representations, the gap between a Euclidean network output space and the non-Euclidean $\SO$ manifold hasn't been completely filled. 
One important yet long-neglected problem lies in optimization on non-Euclidean manifolds\cite{absil2009optimization}:
to optimize on $\SO$ manifold, the optimization variable is a rotation matrix, which contains nine matrix elements; if we naively use \textit{Euclidean gradient}, which simply computes the partial derivatives with respect to each of the nine matrix elements, to update the variable, this optimization step will usually lead to a new matrix off $\SO$ manifold. 
Unfortunately, we observe that all the existing works on rotation regression simply rely upon 
\textit{vanilla auto-differentiation} for backpropagation, exactly computing Euclidean gradient and performing such off-manifold updates to predicted rotations.
We argue that, for training deep rotation regression networks, the off-manifold components will lead to noise in the gradient of neural network weights, hindering network training and convergence.

To tackle this issue, we draw inspiration from differential geometry, where people leverage \emph{Riemannian optimization} to optimize on the non-Euclidean manifold, which finds the direction of the steepest geodesic path on the manifold and take an on-manifold step.
We thus propose to leverage Riemannian optimization and delve deep into the study of the backward pass.
Note that this is a fundamental yet currently under-explored avenue, given that most of the existing works focus on a holistic design of rotation regression that is agnostic to forward/backward pass.
However, incorporating Riemannian optimization into network training is highly non-trivial and challenging. Although methods of Riemannian optimization allow for optimization on $\SO$~\cite{taylor1994minimization, blanco2010tutorial}, matrix manifolds~\cite{absil2009optimization} or general Riemannian manifolds \cite{zhang2016riemannian, udriste2013convex}, they are not directly applicable to update the weights of the neural networks that are Euclidean. Also, approaches like~\cite{hou2018computing} incorporate a Riemannian distance as well as its gradient into network training, however, they do not deal with the \emph{representation} issue.

In this work, we want to \textit{propose a better manifold-aware gradient in the backward pass of rotation regression that directly updates the neural network weights}.
We begin by taking a Riemannian optimization step and computing the difference between the rotation prediction and the updated rotation, which is closer to the ground truth.
Backpropagating this "error", we encounter the mapping function (or orthogonalization function) that transforms the raw network output to a valid rotation. 
This projection, which can be the Gram-Schmidt process or SVD orthogonalization\cite{levinson2020analysis}, is typically a many-to-one mapping.
This non-bijectivity provides us with a new design space for our gradient: if we were to use a gradient to update the raw output rotation, many gradients would result in the same update in the final output rotation despite being completely different for backpropagating into the neural network weights. 
Now the problem becomes: \textit{which gradient is the best for backpropagation when many of them correspond to the same update to the output?}

We observe that this problem is somewhat similar to some problems with ambiguities or multi-ground-truth issues. One example would be the symmetry issue in pose estimation: a symmetric object, \textit{e.g.} a textureless cube, appears the same under many different poses, which needs to be considered when supervising the pose predictions.
For supervising the learning in such a problem, 
Wang \textit{et. al.}\cite{wang2019normalized} proposed to use min-of-N loss\cite{fan2017point}, which only penalizes the smallest error between the prediction and all the possible ground truths.
We therefore propose to find the gradient with the smallest norm that can update the final output rotation to the goal rotation.
This \emph{back-projection} process involves finding an element closest to the network output in the inverse image of the goal rotation and projecting the network output to this inverse image space. We therefore coin our gradient \textit{projective manifold gradient}. 
One thing to note is that this projective gradient tends to shorten the network output, causing the norms of network output to vanish. To fix this problem, we further incorporate a simple regularization into the gradient, leading to our full solution \textit{regularized projective manifold gradient} (RPMG).

Note that our proposed gradient layer operates on the raw network output and can be directly backpropagated into the network weights. Our method is very general and is not tied to a specific rotation representation. It can be coupled with different non-Euclidean rotation representations, including quaternion, 6D representation  \cite{zhou2019continuity}, and 9D rotation matrix representation \cite{levinson2020analysis}, and can even be used for regressing other non-manifold variables.

We evaluate our devised projective manifold gradient layers on a diverse set of problems involving rotation regression: 3D object pose estimation from 3D point clouds/images, rotation estimation problems without using ground truth rotation supervisions, and please see Appendix \ref{sec:reloc} for more experiments on camera relocalization. Our method demonstrates significant and consistent improvements on all these tasks and all different rotation representations tested. 
Going beyond rotation estimation, we also demonstrate performance improvements on regressing unit vectors (lie on a unit sphere) as an example of an extension to other non-Euclidean manifolds.

We summarize our contribution as below: 
\begin{itemize}
    \setlength\itemsep{-1mm}
    \item We propose a novel manifold-aware gradient layer, namely \textit{RPMG}, for the backward pass of rotation regression, which can be applied to different rotation representations and losses and used as a ``plug-in'' at no actual cost. 
    
    \item Our extensive experiments over different tasks and rotation representations demonstrate the significant improvements from using RPMG.
    
    \item Our method can also benefit regression tasks on other manifolds, \textit{e.g.} $\Sphere^2$.
\end{itemize}

\vspace{-1mm}
\section{Related Work}
\vspace{-2mm}
\label{sec:related}

Both rotation parameterization and optimization on SO(3) are well-studied topics.
Early deep learning models leverage various rotation representations for pose estimation, \textit{e.g.}, direction cosine matrix (DCM)~\cite{huang2021multibodysync,yi2019deep}, axis-angle \cite{DeMoN17, deep-6dpose18, occlusion18}, quaternion \cite{xiang2017posecnn, geometric_loss17, posenet,zhao2020quaternion,deng2020deep} and Euler-angle \cite{viewpoints_and_keypoints15, renderforcnn15, 3d-rcnn18}. 
Recently, \cite{zhou2019continuity} points out that Euler-angle, axis-angle, and quaternion are not continuous rotation representations, since their representation spaces are not homeomorphic to SO(3). As better representations for rotation regression, 6D~\cite{zhou2019continuity}, 9D~\cite{levinson2020analysis}, 10D~\cite{peretroukhin_so3_2020} representations are proposed to resolve the discontinuity issue and improve the regression accuracy. A concurrent work \cite{regression_manifold21} examines different manifold mappings theoretically and experimentally, finding out that SVD orthogonalization performs the best when regressing arbitrary rotations.
Originating from general Riemannian optimization, \cite{taylor1994minimization} presents an easy approach for minimization on the $SO(3)$ group by constructing a local axis-angle parameterization, which is also the tangent space of $SO(3)$ manifold. They backpropagate gradient to the tangent space and use the exponential map to update the current rotation matrix.
Most recently, \cite{teed2021tangent} constructs a PyTorch library that supports tangent space gradient backpropagation for 3D transformation groups, (\textit{e.g.}, $SO(3)$, $SE(3)$, $Sim(3)$). This proposed library can be used to implement the Riemannian gradient in our layer.
\vspace{-1mm}
\section{Preliminaries}
\label{sec:prelim}
\vspace{-1mm}
\subsection{Riemannian Geometry}
\vspace{-1mm}

Following~\cite{birdal2018bayesian,birdal2019probabilistic}, we define an $m$-dimensional \textit{Riemannian manifold} embedded in an ambient Euclidean space $\Amb = \R^d$ and endowed with a \textit{Riemannian metric} $\G\triangleq (\Gx)_{\x\in\Man}$ to be a smooth curved space $(\Man,G)$. A vector $\v\in\Amb$ is said to be \emph{tangent} to $\Man$ at $\x$ iff there exists a smooth curve $\curve:[0,1]\mapsto\Man$ s.t. $\curve(0)=\x$ and $\dcurve(0)=\v$. The velocities of all such curves through $\x$ form the \emph{tangent space} $\TxM=\{ \dcurve (0) \,|\, \curve:\R\mapsto\Man \text{ is smooth around $0$ and } \curve(0)=\x\}$.

\vspace{-1mm}
\begin{dfn}[Riemannian gradient]
For a smooth function $f:\Man\mapsto\R$ and $\forall (\x,\v)\in \TM$, we define the \emph{Riemannian gradient} of $f$ as the unique vector field $\grad{f}$ satisfying~\cite{boumal2020introduction}:
\vspace{-2mm}
\begin{equation}
    \mathrm{D}f(\x)[\v] = \langle \v, \grad{f(\x)} \rangle_\x
\vspace{-1mm}
\end{equation}
where $\mathrm{D}f(\x)[\v]$ is the derivation of $f$ by $\v$. It can further be shown (see Appendix \ref{sec:supp2.1}) that an expression for $\grad{f}$ can be obtained through the projection of the \emph{Euclidean} gradient orthogonally onto the tangent space
\vspace{-2mm}
\begin{equation}
\grad{f(\x)} = \nabla f(\x)_{\|} = \Pi_{\x}\big( \nabla f(\x)\big).    
\vspace{-1mm}
\end{equation}
where $\Pi_{\x}:\Amb\mapsto\TxM\subseteq \Amb$ is an orthogonal projector with respect to $\langle \cdot,\cdot \rangle_{\x}$. 
\end{dfn}

\vspace{-2mm}
\begin{dfn}[Riemannian optimization]
We consider gradient descent to solve the problems of $\min_{\x\in\Man}f(\x)$. For a local minimizer or a \emph{stationary point} $\x^\star$ of $f$, the Riemannian gradient vanishes $\grad{f(\x^\star)}=0$ enabling a simple algorithm, \emph{Riemannian gradient descent} (RGD):
\vspace{-2mm}
\begin{align}
\x_{k+1} = R_{\x_k}(-\tau_k\,\grad{f(\x_k)})
\end{align}
where $\tau_k$ is the step size at iteration $k$ and $R_{\x_k}$ is the \emph{retraction} usually chosen related to the exponential map. 
\end{dfn}


\vspace{-2mm}
\subsection{Rotation Representations}
\vspace{-1mm}
\label{sec:rot_rep}
There are many ways of representing a rotation: classic rotation representations, \eg Euler angles, axis-angle, and quaternion; and recently introduced regression-friendly rotation representations such as \eg 5D~\cite{zhou2019continuity}, 6D~\cite{zhou2019continuity}, 9D~\cite{levinson2020analysis} and 10D~\cite{peretroukhin_so3_2020} representations. A majority of deep neural networks can output an \emph{unconstrained}, arbitrary $n$-dimensional vector $\x$ in a Euclidean space $\mathcal{X}=\R^n$. For Euler angle and axis-angle representations which use a vector from $\R^3$ to represent a rotation, a neural network can simply output a 3D vector; however, for quaternions, 6D, 9D or 10D representations that lies on non-Euclidean manifolds, manifold mapping function $\pi:\R^n\mapsto \Man$ is generally needed for normalization or orthogonalization purposes to convert network outputs to valid elements belonging to the representation manifold. This network Euclidean output space $\Amb$ is where the representation manifolds reside and therefore are also called ambient space.

\begin{dfn}[Rotation representation]
One rotation representation, which lies on a representation manifold $\Man$, defines a surjective rotation mapping $\phi: \hat{\x} \in \Man \rightarrow \phi(\hat{\x}) \in \SO$ and a representation mapping function $\psi : \Rot  \in \SO  \rightarrow \psi(\Rot) \in \Man$, such that $\phi(\psi) = \Rot \in \SO$.
\end{dfn}

\begin{dfn}[Manifold mapping function]
From an ambient space $\mathcal{X}$ to the representation manifold $\Man$, we can define a manifold mapping function $\pi: \x \in \mathcal{X} \rightarrow \pi(\x) \in \Man$, which projects a point $\x$ in the ambient, Euclidean space to a valid element $\hat{\x} = \pi(\x)$ on the manifold $\Man$.
\end{dfn}

We summarize the manifold mappings, the rotation mappings and representation mappings for several non-Euclidean rotation representations below.

\noindent\textbf{Unit quaternion.} Unit quaternions represent a rotation using a 4D unit vector $\q \in \Sphere^{3}$ double covering the non-Euclidean $3$-sphere \ie $\q$ and $-\q$ identify the same rotation.
A network with a final linear activation can only predict $\x \in \R^{4}$. The corresponding manifold mapping function is usually chosen to be a normalization step, which reads $\pi_{q}(\x) = \x/\|\x\|$. For rotation and representation mapping, we leverage the standard mappings between rotation and quaternion (see Appendix \ref{sec:q_r_trans}).

\noindent\textbf{6D representation and Gram-Schmidt orthogonalization.} 6D rotation representation\cite{zhou2019continuity}, lying on  Stiefel manifold $\mathcal{V}_2(\mathbb{R}^{3})$, uses two orthogonal unit 3D vectors $(\hat{\mathbf{c}}_1, \hat{\mathbf{c}}_2)$ to represent a rotation, which are essentially the first two columns of a rotation matrix. 
Its manifold mapping $\pi_{6D}$ is done through Gram-Schmidt orthogonalization. Its rotation mapping $\phi_{6D}$ is done by adding the third column $\hat{\mathbf{c}}_3 = \hat{\mathbf{c}}_1 \times \hat{\mathbf{c}}_2$. Its representation mapping  $\psi_{6D}$ is simply getting rid of the third column $\hat{\mathbf{c}}_3$ from a rotation matrix.

\noindent\textbf{9D representation and SVD orthogonalization.} 
To map a raw 9D network output $\M$ to a rotation matrix, \cite{levinson2020analysis} 
use SVD orthogonalization as the manifold mapping function $\pi_{9D}$, as follows: $\pi_{9D}$ first decomposes $\M$ into its left and right singular vectors $\{\U,\V^\top\}$ and singular values $\Sigma$, $\M=\U\Sigma \V^\top$; then it replaces $\Sigma$ with $\Sigma'=\mathrm{diag}(1,1,\mathrm{det}(\U\V^\top))$ and finally, computes $\Rot=\U\Sigma'\V^\top$ to get the corresponding rotation matrix $\Rot \in$ $\SO$. As this representation manifold is $\SO$, both the rotation and representation mapping functions are simply identity. 

\begin{figure*}
  \centering
    \includegraphics[width=0.95\textwidth]{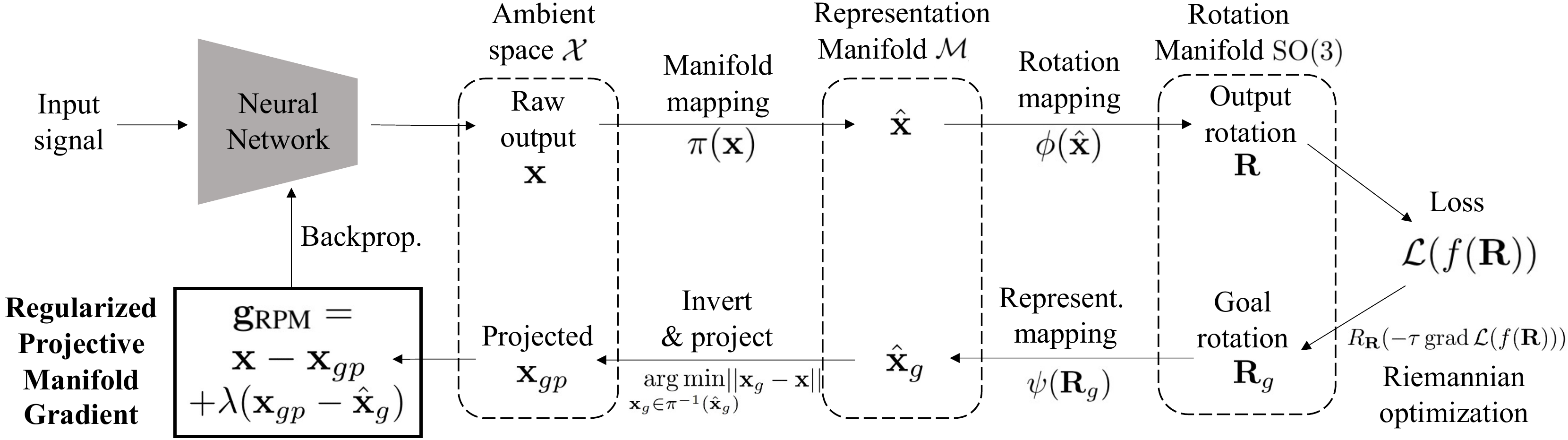}  
    \caption{\textbf{Projective Manifold Gradient Layer.} In the forward pass, the network predicts a raw output $\mathbf{x}$, which is then transformed into a valid rotation $\Rot = \phi(\pi(\x))$. We leave this forward pass unchanged and only modify the backward pass. In the backward pass, we first use Riemannian optimization to get a goal rotation $\mathbf{R}_g$ and map it back to $\hat{\x}_g$ on the representation manifold $\Man$. After that we find the element $\mathbf{x}_{gp}$ which is closest to the raw output in the inverse image of $\hat{\x}_g$, and finally get the gradient $\mathbf{g_{RPM}}$ we want.\vspace{-4mm}}
    \label{fig:pip}
\end{figure*}

\noindent\textbf{10D representation.}
~\cite{peretroukhin_so3_2020} propose a novel 10D representation for rotation matrix. The manifold mapping function $\pi_{10D}$ maps $\boldsymbol{\theta}\in\R^{10}$ to $\q\in\Sphere^3$ by computing the eigenvector corresponding to the smallest eigenvalue of $\mathbf{A}(\boldsymbol{\theta})$, expressed as $\pi_{10D}(\x)=\underset{\q\in\Sphere^3}{\min}~\q^\top\mathbf{A}(\x)\q$, in which 
\begin{equation}
\vspace{1mm}
    \mathbf{A}(\boldsymbol{\theta})~=~\left[ 
        \begin{array}{cccc}
            \theta_1 & \theta_2 & \theta_3 & \theta_4 \\
            \theta_2 & \theta_5 & \theta_6 & \theta_7 \\
            \theta_3 & \theta_6 & \theta_8 & \theta_9 \\
            \theta_4 & \theta_7 & \theta_9 & \theta_{10} \\
        \end{array}
    \right].
    \vspace{1mm}
\end{equation}
Since the representation manifold is $\Sphere^3$, the rotation and representation mapping are the same as unit quaternion.

\subsection{Deep Rotation Regression}
\vspace{-1mm}
We conclude this section by describing the ordinary forward and backward passes of a neural network based rotation regression, as used in \cite{zhou2019continuity, levinson2020analysis}. \vspace{1mm} \\
\textbf{Forward and backward passes.} Assume, for a rotation representation, the network predicts $\x \in \mathcal{X}$, then the manifold mapping $\pi$ will map $\x$ to $\hat{\x} = \pi(\x) \in \Man$, followed by a rotation mapping $\phi$ that finally yields the output rotation $\Rot = \phi(\hat{\x}) = \phi(\pi(\x))$. Our work only tackles the backward pass and keeps the forward pass unchanged, as shown in the top part of Figure \ref{fig:pip}. The gradient in the backward-pass is simply computed using Pytorch autograd method, that is $\mathbf{g}=f'(\Rot)\phi'(\hat{\x})\pi'(\x)$. \vspace{1mm}\\
\textbf{Loss function.}  The most common choice for supervising rotation matrix is L2 loss, $\|\Rot-\Rot_{gt}\|^{2}_{F}$ , as used by \cite{zhou2019continuity, levinson2020analysis}. This loss is equal to $4-4\cos(<\Rot, \Rot_{gt}>)$, where $<\Rot, \Rot_{gt}>$ represents the angle between $\Rot$ and $\Rot_{gt}$. 

\section{Method}
\label{sec:method}
\vspace{-1mm}

\noindent\textbf{Overview.}
In this work, we propose a \emph{projective manifold gradient layer}, without changing the forward pass of a given rotation regressing network, as shown in Figure \ref{fig:pip}. Our focus is to find a better gradient $\g$ of the loss function $\Loss$ with respect to the network raw output $\x$ for backpropagation into the network weights. 

Let's start with examining the gradient of network output $\x$ in a general case -- regression in Euclidean space. Given a ground truth $\x_{gt}$ and the L2 loss $\|\x - \x_{gt}\|^2$ that maximizes the likelihood in the presence of Gaussian noise in $\x$, the gradient would be $\g = 2(\x - \x_{gt})$.

In the case of rotation regression, we therefore propose to find a proper $\x^*\in \mathcal{X}$ for a given ground truth $\Rot_{gt}$ or a computed goal rotation $\Rot_{g}$ when the ground truth rotation is not available, and then simply use $\x - \x^*$ as our gradient to backpropagate into the network.

Note that finding such an $\x^*$ can be challenging. 
Assuming we know $\Rot_{gt}$,  finding an $\x^*$ involves inverting $\phi$ and $\pi$ since the network output $\Rot = \phi(\pi(\x))$.
Furthermore, we may not know $\Rot_{gt}$ under indirect rotation supervision (\textit{e.g.}, flow loss as used in PoseCNN\cite{xiang2017posecnn}) and self-supervised rotation estimation cases (\textit{e.g.}, 2D mask loss as used in \cite{wang2020self6d}). 

In this work, we introduce the following techniques to mitigate these problems: (i) we first take a Riemannian gradient to compute a goal rotation $\Rot_{g} \in \SO$, which does not rely on knowing $\Rot_{gt}$, as explained in Section \ref{sec:rg}; (ii) we then find the set of all possible $\x_g$s that can be mapped to $\Rot_{g}$, or in other words, the inverse image of $\Rot_{g}$ under $\pi$ and $\phi$; (iii) we find $\x_{gp}$ which is the element in this set closest to $\x$ in the Euclidean metric and set it as ``$\x^*$''. We will construct our gradient using this $\x^*$, as explained in \ref{sec:pmg}. (iv) we add a regularization term to this gradient forming $\g_{RPMG}$ as explained in \ref{sec:rpmg}. The whole backward pass leveraging our proposed regularized projective manifold gradient is shown in the lower half of Figure \ref{fig:pip}.

\vspace{-1mm}
\subsection{Riemannian Gradient and Goal Rotation}
\vspace{-1mm}
\label{sec:rg}
To handle rotation estimation with/without direct rotation supervision, we first propose to compute the Riemannian gradient of the loss function $\Loss$ with respect to the output rotation $\Rot$ and find a goal rotation $\Rot_{g}$ that is presumably closer to the ground truth rotation than $\Rot$.

Assume the loss function is in the following form $\Loss(f(\Rot))$, where $\Rot = \pi(\phi(\x))$ is the output rotation and $f$ constructs a loss function that compares $\Rot$ to the ground truth rotation $\Rot_{gt}$ directly or indirectly. Given $\Rot(\x)$ and $\Loss(f(\Rot(\x)))$, we can perform one step of Riemannian optimization yielding our goal rotation $\Rot_{g} \gets R_{\Rot}(-\tau\,
\grad{~\Loss(f(\Rot}))),$
where $\tau$ is the step size of Riemannian gradient and can be set to a constant as a hyperparameter or varying during the training. For L2 loss $\|\Rot - \Rot_{gt}\|_F^2$, the Riemannian gradient is always along the geodesic path between ${\Rot}$ and $\Rot_{gt}$ on SO(3)\cite{huynh2009metrics}. In this case, $\Rot_{g}$ can generally be seen as an intermediate goal between $\Rot$ and $\Rot_{gt}$ dependent on $\tau$.
Gradually increasing $\tau$ from 0 will first make $\Rot_{g}$ approach $\Rot_{gt}$ starting with $\Rot_{g}=\Rot$, and then reach $\Rot_{gt}$ where we denote $\tau = \tau_{gt}$, and finally going beyond $\Rot_{gt}$.
Although, when $\Rot_{gt}$ is available, one can simply set $\Rot_{g} = \Rot_{gt}$, we argue that this is just a special case under $\tau =\tau_{gt}$. 
For scenarios where $\Rot_{gt}$ is unavailable, \textit{e.g.}, in self-supervised learning cases (see in Section \ref{sec:wo_gt}), we don't know $\Rot_{gt}$ and $\tau_{gt}$, thus we need to compute $\Rot_{g}$ using Riemannian optimization.
In the sequel, we only use $\Rot_{g}$ for explaining our methods without loss of generality.
See Section \ref{sec:rpmg} for how to choose $\tau$.

\subsection{Projective Manifold Gradient}
\vspace{-1mm}
\label{sec:pmg}

Given $\Rot_{g}$, we can use the representation mapping $\psi$ to find the corresponding $\hat{\x}_g = \psi(\Rot_g)$ on the representation manifold $\Man$. However, further inverting $\pi$ and finding the corresponding $\x_g \in \mathcal{X}$ is a non-trivial problem, due to the projective nature of $\pi$.
In fact, there are many $\x_g$s that satisfy $\pi(\x_g) = \hat{\x}_g$.
It seems that we can construct a gradient $\mathbf{g} = (\x - \x_g)$ using any $\x_g$ that satisfies $\pi(\x_g) = \hat{\x}_g$. No matter which $\x_g$ we choose, if this gradient were to update $\x$, it will result in the same $\Rot_g$. 
But, when backpropagating into the network, those gradients will update the network weights differently, potentially resulting in different learning efficiency and network performance. 

Formally, we formulate this problem as \emph{a multi-ground-truth problem} for $\x$: we need to find the best $\x^*$ to supervise from the inverse image of $\hat{\x}_g$ under the mapping $\pi$. 
We note that similar problems have been seen in pose supervision dealing with symmetry as in \cite{wang2019normalized}, where one needs to find one pose to supervise when there are many poses under which the object appears the same. \cite{wang2019normalized} proposed to use a min-of-N strategy introduced by \cite{fan2017point}: from all possible poses, taking the pose that is closest to the network prediction as ground truth. A similar strategy is also seen in supervising quaternion regression, as $\q$ and $-\q$ stand for the same rotation. One common choice of the loss function is therefore $\min \{\mathcal{L}(\q, \q_{gt}), \mathcal{L}(\q,-\q_{gt})\}$\cite{peretroukhin_so3_2020}, which penalizes the distance to the closest ground truth quaternion.

Inspired by these works, we propose to choose our gradient among all the possible gradients with the lowest level of redundancy, \textit{i.e.},  we require
$\x^*$ to be the one closest to $\x$, or in other words, the gradient to have the smallest norm, meaning that we need to find the projection point $\x_{gp}$ of $\x$ to all the valid $\x_g$:
\begin{equation}
    \x_{gp} = \underset{\pi(\x_g) = \hat{\x}_g}{\text{argmin}}~\|\x - \x_g\|_2
\end{equation}
We then can construct our \textit{projective manifold gradient} (PMG) as $\g_{PM} = \x - \x_{gp}$. We will denote the naive gradient $\g_{M} = \x - \hat{\x}_{g}$ as  \textit{manifold gradient} (MG).

Here we provide another perspective on why a network may prefer PMG. In the case where a deep network is trained using stochastic gradient descent (SGD), the final gradient used to update the network weights is averaged across the gradients of all the batch instances. 
If gradients from different batch instances contain different levels of redundancy, then the averaged gradient may be biased or not even appropriate.
This argument is generally applicable to all stochastic optimizers (\textit{e.g.}, Adam~\cite{Adams2020AFD}) 

\begin{figure}
\hspace{2mm}
\begin{minipage}{0.4\linewidth}
    \centering
    \includegraphics[width=1\linewidth]{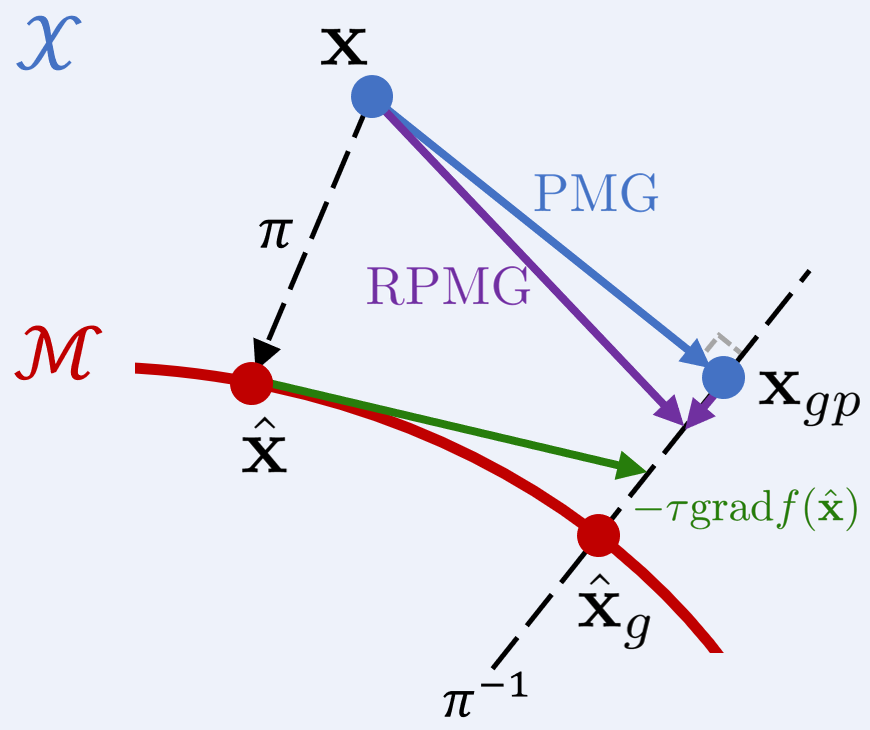}
     
\end{minipage}\hfill 
\begin{minipage}{0.45\linewidth}
    \centering

 \includegraphics[width=1\columnwidth]{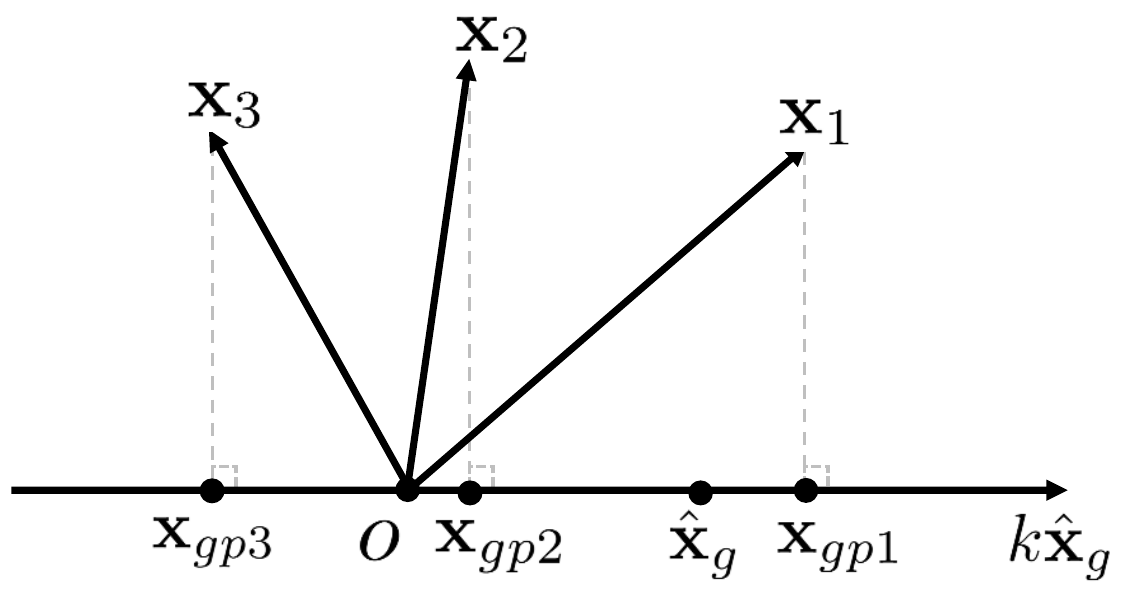}
 \end{minipage}
 \hspace{2mm}
\caption{\textbf{Illustration for regularized projective manifold gradient.} \textbf{Left}: In the forward pass, we simply project $\x$ to $\hat{\x}$ by $\pi$. In the backward pass, first we compute a Riemannian gradient, which is shown as the \textit{green} arrow. After getting a next goal $\hat{\x}_{g}\in\Man$ by Riemannian optimization, we find the inverse projection $\x_{gp}$ of $\hat{\x}_g$, which leads to our \textit{projective manifold gradient}, shown as the \textit{blue} arrow. With a regularization term, we can get our final \textit{regularized projective manifold gradient}, as the \textit{purple} arrow. \textbf{Right}: Projection point $\hat{\x}_{gp}$ in the case of quaternion.}
     \vspace{-1mm}
    \label{fig:reg}
\end{figure}

\noindent\textbf{Inverting $\pi$.} There are many ways to solve this projection problem for different manifold mapping functions $\pi$. For example, we can formulate this as a constrained optimization problem. For the manifold mapping functions we consider, we propose the following approach: we first solve for the inverse image $\pi^{-1}(\hat{\x}_g)$ of $\hat{\x}_g$ in the ambient space $\mathcal{X}$ analytically, which reads
$\pi^{-1}(\hat{\x}_g) = \{\x_g \in \mathcal{X}~|~\pi(\x_g) = \hat{\x}_g\}$; we then project $\x$ onto this inverse image space.
Note that, sometimes only a superset of this inverse image can be found analytically, requiring certain constraints on $\x_{gp}$ to be enforced.

Here we list the inverse image $\pi^{-1}(\hat{\x}_g)$ and the projection point $\x_{gp}$ for different rotation representations and their corresponding manifold mapping $\pi$. Please refer to Appendix \ref{sec:inverse_proj} for detailed derivations.

\noindent\textbf{Quaternion.} With $\pi_q(\x) = \x/{\|\x\|}$, $\x \in \R^4$, and $\hat{\x}_g \in \mathcal{S}^{3}$: $\pi^{-1}_{q}(\hat{\x}_g) = \{\x~|~\x = k\hat{\x}_g, k\in \mathbb{R}~\text{and}~k>0\}$, which is a ray in the direction of $\hat{\x}_g$ starting from the origin. Without considering the constraint of $k>0$ , an analytical solution to this projection point $\x_{gp}$ of $\x$ onto this line can be derived: $\x_{gp} = (\x\cdot\hat{\x}_g)\hat{\x}_g$.

\noindent\textbf{6D representation.} 
With $\pi_{6D}$ as Gram-Schmidt process, $\x=[\u, \v] \in \R^6$, and $\hat{\x_g}\in \mathcal{V}_2(\mathbb{R}^{3})$:  $\pi^{-1}_{6D}(\hat{\x}_g) = \{[k_1\hat{\u}_g, k_2\hat{\u}_g+k_3\hat{\v}_g]~|~k_1,k_2,k_3 \in \R~\text{and}~k_1, k_3>0\}$ (the former is a ray whereas the latter spans a half plane). Without considering the constraint of $k_1,k_3>0$, the projection point $\x_{gp}$ can be analytically represented as  $\x_{gp} = [(\u\cdot\hat{\u}_g)\hat{\u}_g, (\v\cdot\hat{\u}_g)\hat{\u}_g+(\v\cdot\hat{\v}_g)\hat{\v}_g]$

\noindent\textbf{9D representation.} 
With $\pi_{9D}(\x)$ as SVD orthogonalization, $\x \in \R^{3\times 3}$, and $\hat{\x}_g\in$ SO(3), the analytical expression for $\pi_{9D}^{-1}$ is available when we ignore the positive singular value constraints, which gives $\pi^{-1}_{9D}(\hat{\x}_g) = \{\mathbf{S}\hat{\x}_g~|~\mathbf{S}=\mathbf{S^\top}\}$. We can further solve the projection point $\x_{gp}$ with an elegant representation $\x_{gp} = \frac{\x\hat{\x}_g^{T}+\hat{\x}_g\x^{T}}{2}$.

\noindent\textbf{10D representation.}
Please refer to Appendix \ref{sec:inverse_proj} for the derivation and expression of $\x_{qp}$.

\vspace{-1mm}
\subsection{Regularized Projective Manifold Gradient}
\vspace{-1mm}

\label{sec:rpmg}

\noindent\textbf{Issues in naive projective manifold gradient.}
\label{issues}
In the right plot of Figure \ref{fig:reg}, we illustrate this projection process for several occasions where $\x$ takes different positions relative to $\x_g$. 
We demonstrate that there are two issues in this process. 

First, no matter where $\x$ is in, the projection operation will shorten the length of our prediction because $\|\x_{gp}\| < \|\x\|$ is always true for all of 4D/6D/9D/10D representation. This will cause the length norm of the prediction of the network to become very small as the training progresses (see Figure \ref{fig:length}). The shrinking network output will keep increasing the effective learning rate, preventing the network from convergence and leading to great harm to the network performance (see Table \ref{tab:ablation} and Figure \ref{fig:length} for ablation study).

Second, when the angle between $\x$ and $\hat{\x}_g$ becomes larger than $\pi/2$ (in the case of $\x = \x_3$), the naive projection $\x_{gp}$ will be in the opposite direction of $\hat{\x}_g$ and can not be mapped back to $\hat{\x}_g$ under $\pi_{q}$, resulting in a wrong gradient. The same set of issues also happens to 6D, 9D and 10D representations. The formal reason is that the analytical solution of the inverse image assumes certain constraints are satisfied, which is usually true only when either $\hat{\x}_g$ is not far from $\x$ or the network is about to converge. 

\noindent\textbf{Regularized projective manifold gradient.}
To solve the first issue, we propose to add a regularization term $\x_{gp} - \hat{\x}_g$ to the projective manifold gradient, which can avoid the length vanishing problem. The \textit{regularized projective manifold gradient} then reads:
\begin{equation}
  \g_{RPM} = \x - \x_{gp} + \lambda(\x_{gp} - \hat{\x}_g),  
\end{equation}
where $\lambda$ is a regularization coefficient. See the left plot of Figure \ref{fig:reg} for an illustration.

\noindent\textbf{Discussion on the hyperparameters $\lambda$ and $\tau$.}
Our method apparently introduces two additional hyperparameters, $\lambda$ and $\tau$, however, we argue that this doesn't increase the searching space of hyperparameters for our method.

For $\lambda$, the only requirement is that  $\lambda$ is small (we simply set to 0.01), because: (1) we want the projective manifold gradient ($\x - \x_{gp}$) to be the major component of our gradient; (2) since this regularization is roughly proportional to the difference in prediction length and a constant, a small lambda is enough to prevent the length from vanishing and, in the end, the prediction length will stay roughly constant at the equilibrium under projection and regularization. In the ablation study of Section \ref{sec:pc_rotation}, we show that the performance is robust to the change of $\lambda$. Note that, on the other extreme, when $\lambda = 1$, $\g_{RPM}$ becomes $\g_{M}$.

For $\tau$, we propose a ramping up schedule which is well-motivated. 
To tackle the second problem of reversed gradient, we need a small $\tau_{init}$ to keep $\Rot_g$ close to $\Rot$ at the beginning of training. But when the network is about to converge, we will prefer a $\tau_{converge}$ which can keep $\Rot_g$ close to $\Rot_{gt}$ for better convergence. We cannot directly set $\tau_{converge}$ to $\tau_{\gt}$, which is introduced in \ref{sec:rg}, because $\tau_{\gt}$ is not a constant and cannot be used in Riemannian Optimization. However, if we want to tackle the problem of reversed gradient, we must need Riemannian Optimization and $\tau_{init}$. Thus we need a constant approximation of $\tau_{\gt}$ when the angle between $\Rot$ and $\Rot_{gt}$ converges to 0. Note that $\tau_{converge}$ can be derived analytically when the loss function is the most widely used L2 loss or geodesic loss(please refer to Appendix \ref{sec:supp2.1} for details), and therefore doesn't need to be tuned.  Therefore we propose to increase $\tau$ from a small value $\tau_{init}$, leading to a slow warm-up and, as the training progresses, we gradually increase it to the final $\tau=\tau_{converge}$ by ten uniform steps.  This strategy further improves our performance.

\vspace{-2mm}
\section{Experiments}
\vspace{-1mm}
\label{sec:exp}

\begin{table*}[htbp]
    \caption{\textbf{Pose estimation from ModelNet40 point clouds.} Left: a comparison of methods by mean, median, and 5$^\circ$ accuracy of (geodesic) errors after 30k training steps. Mn, Md and Acc are abbreviations of mean, median and 5$^\circ$ accuracy.  Right: median test error of \textit{airplane} in different iterations during training. }
    \vspace{-4mm}
    \begin{minipage}[h]{0.7\columnwidth}
        \resizebox{\columnwidth}{!}{
        \begin{tabular}{l|ccc|ccc|ccc|ccc|ccc}
           \multirow{2}{*}{Methods} &&Airplane&&&Chair&&&Sofa &&&Toilet&&&Bed&\\
              \cmidrule{2-16}   
              & Mn$\downarrow$ & Md$\downarrow$ & Acc$\uparrow$ 
              & Mn$\downarrow$ & Md$\downarrow$ & Acc$\uparrow$ 
              & Mn$\downarrow$ & Md$\downarrow$ & Acc$\uparrow$ 
              & Mn$\downarrow$ & Md$\downarrow$ & Acc$\uparrow$ 
              & Mn$\downarrow$ & Md$\downarrow$ & Acc$\uparrow$  \\
                \midrule   
             Euler       &       125 & 131& 0 &13.6&9.0&17&120&125&0&127&133&0&113&122&0  \\
             Axis-Angle   &      10.8 & 8.2 & 22  &16.4&10.9&9 &24.1&14.6&6&21.9&13.0&9&25.5&11.0&16 \\
            Quaternion & 9.7 & 7.6  & 27 & 16.7& 11.4& 12&20.4& 12.7& 10& 16.0& 9.3& 17 &27.8& 11.3& 14\\
            6D & 5.5 & 4.7 &54 &9.8& 6.4& 35&14.6& 9.5& 15&9.3& 6.8& 33&24.7& 9.6& 17\\
            9D &  4.7 & 3.9 &67 & 7.9& 5.4& 44&15.7& 10.0& 14&10.3& 6.9& 30&22.3& 8.5& 20\\
            9D-Inf~(MG-9D) &3.1 & 2.5 &90 & 5.3& 3.7& 69&7.8& 5.0& 50 &4.2& 3.3& 75 &12.9& 4.6& 55 \\
            10D & 5.3 & 4.2 & 61  &8.9& 6.0& 38&15.1& 10.3& 13&10.7& 6.5& 35& 23.1& 8.7& 19\\
            \midrule
            RPMG-Quat   & 3.2 & 2.4 & 88   &6.3 &3.7 &67 &8.1& 4.5& 57&4.9& 3.5& 74&13.3& 3.6& 70\\
            RPMG-6D    &  2.6 & 2.1&\textbf{94}   &\textbf{5.0} &\textbf{3.1}& 74 &6.6& 3.6& 70&\textbf{3.8}& 2.9& \textbf{83}& 13.5& 2.7& 81\\
            RPMG-9D   & \textbf{2.5} & \textbf{2.0}&\textbf{94}    &    5.1&\textbf{3.1} & \textbf{76} &\textbf{6.1}& \textbf{3.1}& \textbf{77}&4.3& \textbf{2.7}& \textbf{83}&
            \textbf{10.9}& \textbf{2.5}& \textbf{86}\\
            RPMG-10D  & 2.8 & 2.2 & 93 &5.1 &3.2 &75&  6.5& 3.2& 72&4.9& 2.8& 82&13.5& 2.7& 82\\
            \midrule 
            \end{tabular}
            }
           
        \end{minipage}
    \begin{minipage}[h]{0.29\columnwidth}
        \includegraphics[width=\columnwidth]{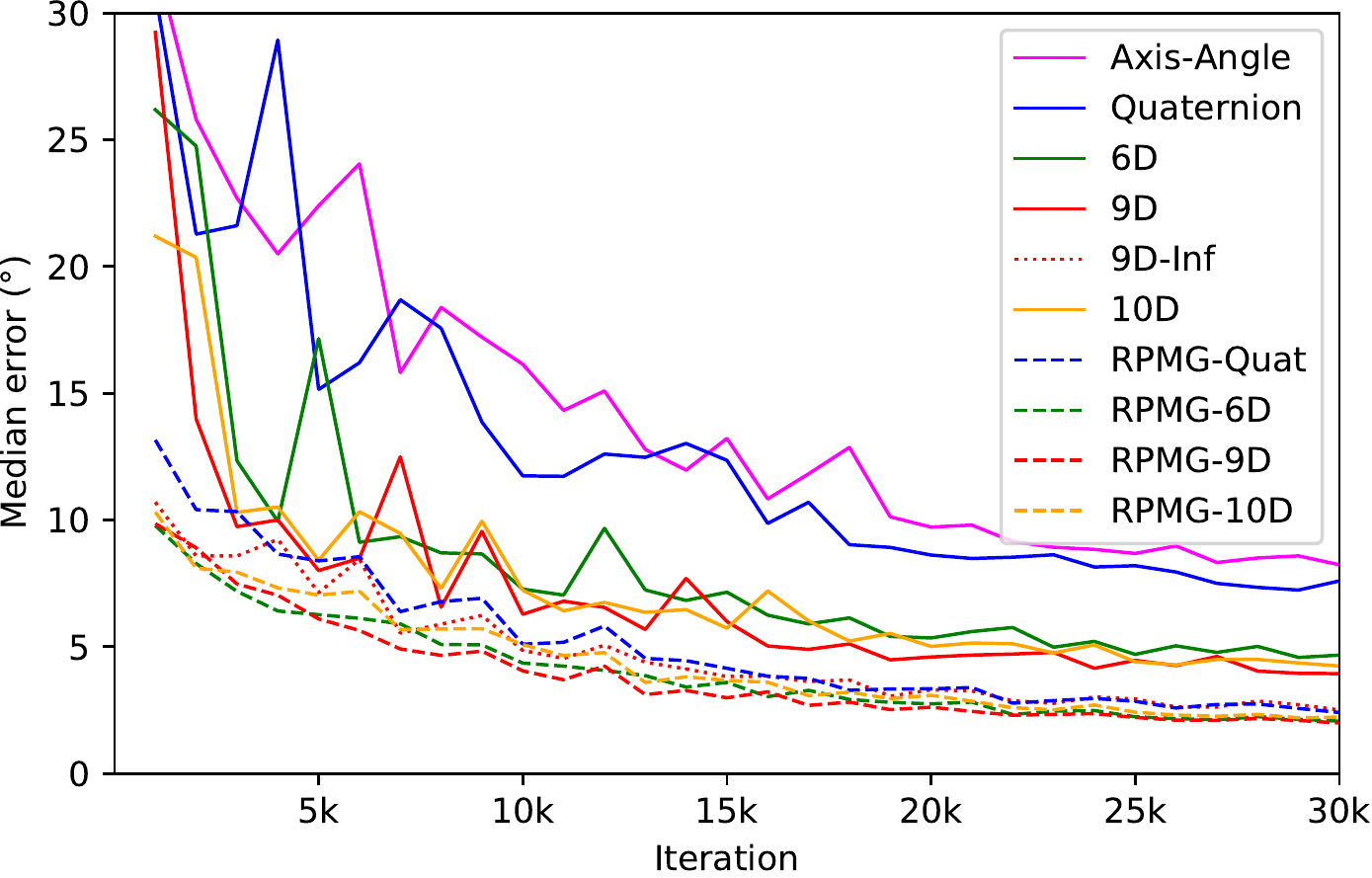}
    \end{minipage}
  \label{tab:rotation_complete}
\end{table*}

We investigate popular rotation representations and find our methods greatly improve the performance in different kinds of tasks. For our regularized projective manifold gradient (\textbf{RPMG}), we apply it in the backpropagation process of Quaternion, 6D, 9D and 10D, without changing the forward pass, leading to three new methods \textbf{RPMG-Quat}, \textbf{RPMG-6D}, \textbf{RPMG-9D} and \textbf{RPMG-10D}.
We  compare the following seven baselines: \textbf{Euler angle}, \textbf{axis-angle}, \textbf{Quaternion}, \textbf{6D} \cite{zhou2019continuity}, \textbf{9D} \cite{levinson2020analysis}, \textbf{9D-Inf} \cite{levinson2020analysis} and \textbf{10D} \cite{peretroukhin_so3_2020}. 
We adopt three evaluation metrics: mean, median, and 5$^\circ$ accuracy of (geodesic) errors between predicted rotation and ground truth rotation.
For most of our experiments, we set the regularization term $\lambda=0.01$ and increase $\tau$ from $\tau_{init}=0.05$ to $\tau_{converge}=0.25$ by ten uniform steps. We further show and discuss the influence of different choices of these two hyperparameters in our ablation studies.

\vspace{-1mm}
\subsection{3D Object Pose Estimation from Point Clouds}
\vspace{-1mm}
\label{sec:pc_rotation}
\noindent\textbf{Experimental setting.} As in \cite{chen2021equivariant}, we use the complete point clouds generated from the models in ModelNet-40 \cite{wu2015modelnet}. We use the same train/test split as in \cite{chen2021equivariant} and report the results of \textit{airplane}, \textit{chair}, \textit{sofa}, \textit{toilet} and \textit{bed} those five categories because they exhibit less rotational symmetries.
Given one shape point clouds of a specific category, the network learns to predict the 3D rotation of the input point clouds from the predefined canonical view of this category\cite{wang2019normalized}. 
We replace the point clouds alignment task used in \cite{zhou2019continuity, levinson2020analysis} (which has almost been solved) by this experiment since it is more challenging and closer to real-world applications (no canonical point clouds is given to the network).

We use a PointNet++ \cite{qi2017pointnetplusplus} network as our backbone, supervised by L2 loss between the predicted rotation matrix $\Rot$ and the ground truth rotation matrix $\Rot_{gt}$.  To facilitate a fair comparison between multiple methods, we use the same set of hyperparameters in all the experiments. Please see Appendix \ref{sec:imp_detail} for more details.

\noindent\textbf{Analysis of results.}
The results are shown in Table \ref{tab:rotation_complete}. We see a great improvement of our methods in all three rotation representations. In this experiment, one may find \textbf{9D-Inf} also leads to a good performance, which is actually a special case of our method with $\lambda=1$, or in other words, it is MG with $\tau=\tau_{gt}$. Nonetheless, in Table \ref{tab:modelnet_all}, we can observe a larger gap. Also, this simple loss may lead to bad performance when $\Rot_{gt}$ is unavailable in Section \ref{sec:self-supervised}. 

\begin{figure}[t]
\vspace{-4mm}
\begin{minipage}{0.5\linewidth}
    \centering
    \includegraphics[width=1\linewidth]{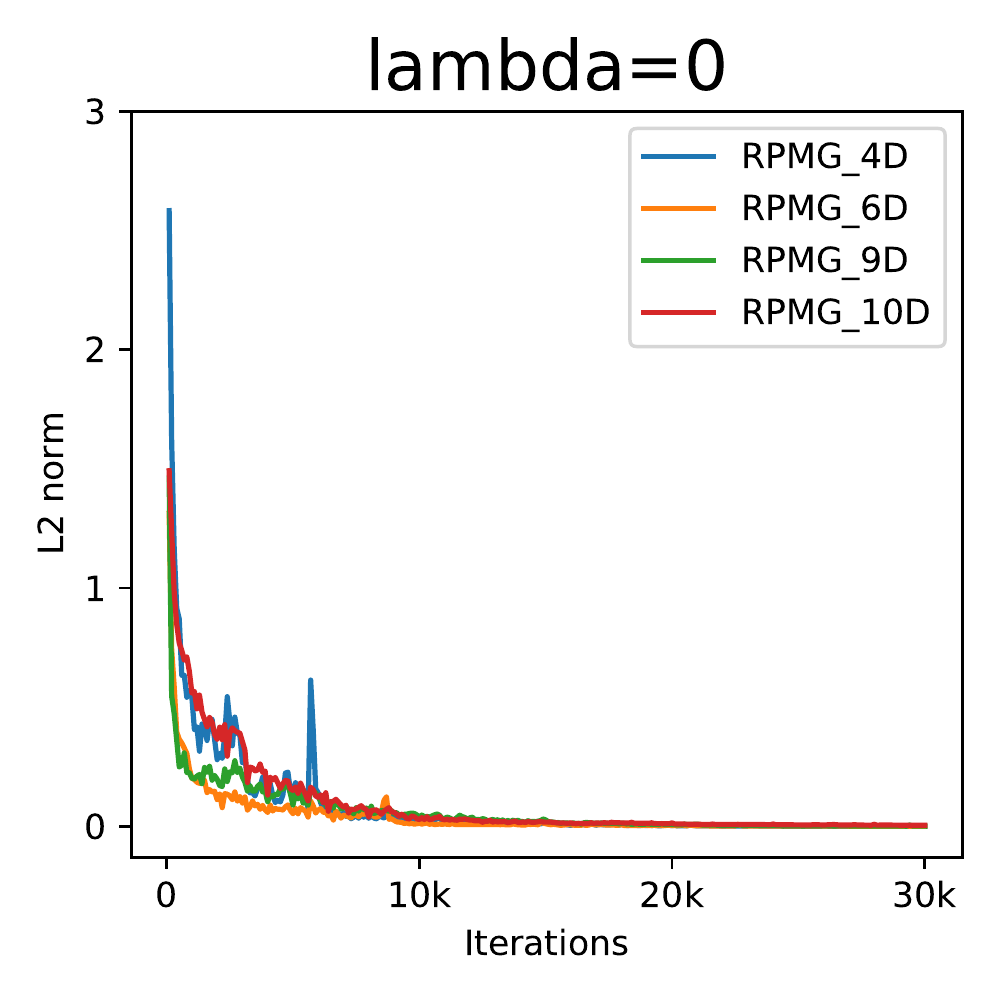}
     
\end{minipage}\hfill 
\begin{minipage}{0.5\linewidth}
    \centering
     \includegraphics[width=1\linewidth]{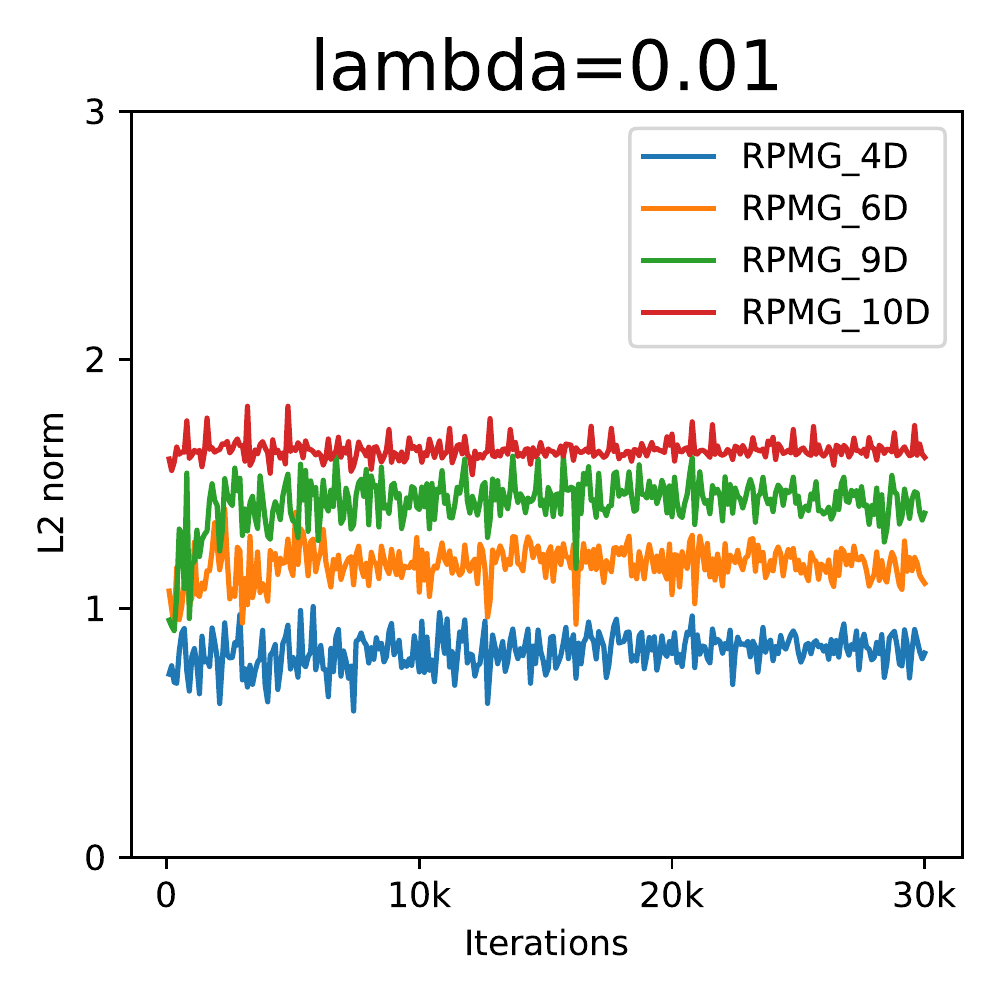}
    \caption{\textbf{Average L2 norm of the network raw output $\x$} during training. Left: PMG-4D/6D/9D/10D (w/o reg. $\lambda=0$). Right: RPMG-4D/6D/9D/10D (w/ reg. $\lambda=0.01$)}
    \label{fig:length}
\end{minipage}
\vspace{-3mm}
\end{figure}

\noindent\textbf{Ablation study on $\lambda$.} As mentioned in Section \ref{sec:rpmg}, naively using \textbf{PMG} without any regularization, corresponding to setting $\lambda = 0$, will lead to length vanishing; To maintain the length of prediction roughly constant, we only need to add a small $\lambda$. In Figure \ref{fig:length}, We show the length vanishing problem without regularization and stabilized length with a small regularization. In Table \ref{tab:rotation_complete}, we show that the network performs much better when we have a small $\lambda$ (\textbf{RPMG}) than $\lambda=0$ (\textbf{PMG}) or $\lambda=1$ (\textbf{MG}), which deviates too far away from the desired projective manifold gradient.  As for the exact value of $\lambda$, our experiments show that our method is robust to the choice of $\lambda$ as long as it is small. Table \ref{tab:ablation} also shows that $\lambda= 0.01, 0.005, 0.05$ all lead to similar performance, thus freeing us from tuning the parameter $\lambda$.

\noindent\textbf{Ablation study on $\tau$.}
For the choices of $\tau$, Table \ref{tab:ablation} shows that our proposed strategy, which ramps up $\tau$ from a small $\tau_\text{init}$ to  $\tau_\text{converge}$, works the best. The reason is that: a big $\tau$, when training begins, may cause the problem of reversed gradient discussed in Section \ref{issues}. On the other side, a small $\tau$ at the end of training will slow down the training process and can do harm to convergence.
Note that, the performance is not very sensitive to the exact value,
which means we don't require a parameter tuning for $\tau$ even in general cases. We are good even with simply setting $\tau = \tau_{gt}$.

\begin{table}[htbp]
    \centering
    \vspace{-2mm}
    \caption{\textbf{Ablation study of pose estimation from \textit{airplane} point clouds.} Here MG stands for manifold gradient $\x-\hat{\x}_g$, corresponding to set $\lambda=1$; PMG stands for projective manifold gradient $\x-\x_{gp}$, corresponding to set $\lambda=0$.}
      
    \begin{minipage}[h]{\columnwidth}
  \resizebox{\columnwidth}{!}{
    \begin{tabular}{p{40pt}ccccc}
         \multicolumn{3}{c}{Methods}&\multicolumn{1}{c}{Mean ($^\circ$)$\downarrow$} & \multicolumn{1}{c}{Med ($^\circ$)$\downarrow$} & \multicolumn{1}{c}{5$^\circ$Acc ($\%$)$\uparrow$} \\
                \midrule
                L2 6D &-&-
                &5.50&4.67&54.4\\
                \midrule
                \multirow{2}{*}{MG-6D} &\multirow{2}{*}{$\lambda=1$}&$\tau_{converge}$ 
                &3.51&2.95&85.2
                \\  
                &&$\tau_{gt}$ 
                &3.19 &2.72 &87.8
                \\
                \midrule
                 \multirow{2}{*}{PMG-6D} &\multirow{2}{*}{$\lambda=0$}&$\tau_{converge}$
               
                 &57.65&45.22&0.2
                 \\
                 &&$\tau_{gt}$
                 &133&136&0.0
                 \\
                 \midrule
                 \multirow{6}{*}{RPMG-6D} &\multirow{4}{*}{$\lambda=0.01$}&$\tau_{init}$
                 &2.67&2.18&93.1
                 \\
                 &&$\tau_{converge}$
             &  2.71&2.14&93.2 
                 \\
                 &&$\tau_{\gt}$&
               3.02 & 2.14 & 89.5 
                 \\
                 &&$\tau_{init}\rightarrow{}\tau_{converge}$
                &2.59&2.07&93.6
                 \\
                 \cmidrule{2-6}
             &$\lambda=0.05$&\multirow{2}{*}{$\tau_{init}\rightarrow{}\tau_{converge}$}& 
                2.73 & 2.23 & 92.9
                 \\
                 &$\lambda=0.005$ &
              &\textbf{2.52}&\textbf{2.05}  & \textbf{94.3}
                 \\
                 
                 \midrule
    \end{tabular}}
    \end{minipage}
    \label{tab:ablation}
     
\end{table}

\begin{table*}[htbp]
    \caption{\textbf{Pose estimation from ModelNet10 images.} Left: a comparison of methods by mean($^\circ$), median($^\circ$), and 5$^\circ$ accuracy($\%$) of (geodesic) errors after 600k training steps. Mn, Md and Acc are abbreviations of mean, median and 5$^\circ$ accuracy.  Right: median test error of \textit{chair} in different iterations during training.}
    \begin{minipage}[h]{0.65\columnwidth}
        \resizebox{\columnwidth}{!}{
            \begin{tabular}{l|ccc|ccc|ccc|ccc}
           
           \multirow{2}{*}{Methods}&& Chair &&&Sofa&&& Toilet&&&Bed&\\
            \cmidrule{2-13}  
                & Mn$\downarrow$ & Md$\downarrow$ & Acc$\uparrow$ 
                & Mn$\downarrow$ & Md$\downarrow$ & Acc$\uparrow$ & Mn$\downarrow$ & Md$\downarrow$ & Acc$\uparrow$ & Mn$\downarrow$ & Md$\downarrow$ & Acc$\uparrow$ 
                \\
                \midrule   
                Euler & 21.5 & 10.9 & 10 & 27.5 & 12.0 & 9  & 14.9 & 8.5 & 19 &27.6 & 9.6 & 17 \\
                Axis-Angle & 25.7 & 14.3 & 7 & 30.3 & 14.6 & 6 & 20.3 & 13.0 & 8 & 36.3&16.7&4   \\
                Quaternion & 25.8 & 15.0 & 6 & 30.0 & 15.7 & 6  &   20.6&13.0&8&34.1&15.5&5    \\
                6D & 19.6 & 9.1  & 19 & 17.5 & 7.3 & 27  &    10.9&6.2&37&32.3&11.7&11     \\
                9D & 17.5 & 8.3 & 23 & 19.8 & 7.6 & 25  &    11.8&6.5&34&30.4&11.1&13     \\
                9D-Inf & 12.1 & 5.1 & 49& 12.5 & 3.5 & 70 & 7.6&3.7&67&22.5&4.5&56     \\
                10D & 18.4 & 9.0 & 20 & 20.9 & 8.7 & 20 &  11.5& 5.9&39& 29.9&11.5&11      \\
                \midrule
                RPMG-Quat &  13.0 & 5.9 & 40 & 13.0 & 3.6 & 67  &   8.6&4.2&61& 23.2&4.9&51    \\
                RPMG-6D & 12.9 & 4.7 & 53 & 11.5 & 2.8 & 77 &      7.8&3.4&71&20.3&3.6&67     \\
                RPMG-9D & \textbf{11.9} & \textbf{4.4} & \textbf{58} & \textbf{10.5} & \textbf{2.4} & \textbf{82}  &     7.5&3.2&75&20.0&\textbf{2.9}&\textbf{76}    \\
                 RPMG-10D & 12.8& 4.5& 55 & 11.2&\textbf{2.4}&\textbf{82}& \textbf{7.2}&\textbf{3.0}&\textbf{76}&\textbf{19.2}&\textbf{2.9}&75\\
                  \midrule
            \end{tabular}}
        \end{minipage}
    \begin{minipage}[h]{0.34\columnwidth}
        \includegraphics[width=\columnwidth]{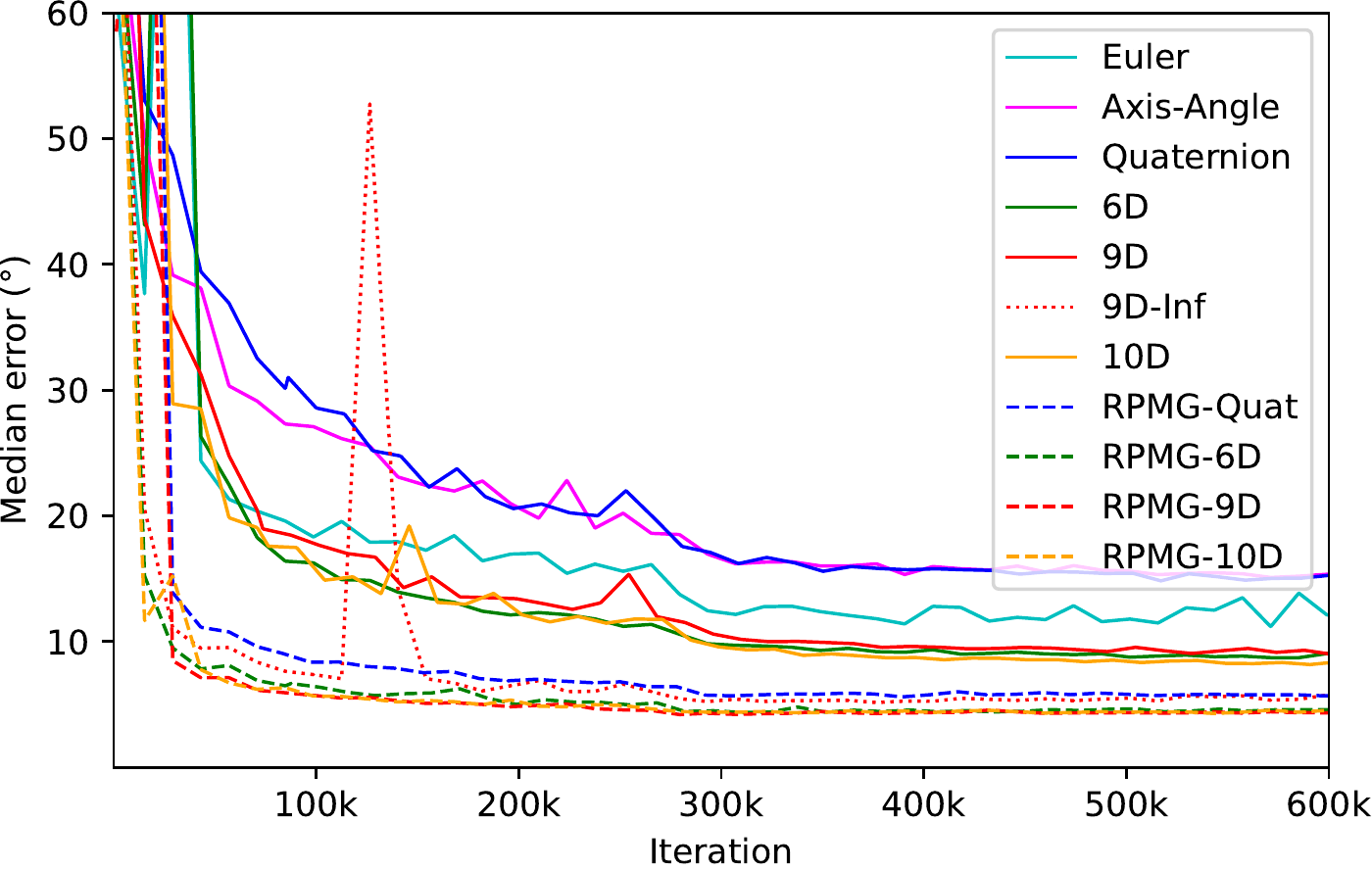}
    \end{minipage}
  \label{tab:modelnet_all}
\end{table*}
\vspace{-3mm}
\subsection{3D Rotation Estimation from ModelNet Images}
\label{sec:image}

In this experiment, we follow the setting in \cite{levinson2020analysis} to estimate poses from 2D images. Images are rendered from ModelNet-10 \cite{wu2015modelnet} objects from arbitrary viewpoints \cite{liao19sphere}. A MobileNet \cite{mobilenet} is used to extract image features and three MLPs to regress rotations.
We use the same categories as in Experiment \ref{sec:pc_rotation} except \textit{airplane}, since ModelNet-10 doesn't have this category.
We didn't quote the numbers from \cite{levinson2020analysis} since we conduct all the experiments using the same set of hyperparameters to ensure a fair comparison. Please see Appendix \ref{sec:img_lr} for more details.

The results are shown in Table \ref{tab:modelnet_all}. Our RPMG layer boosts the performance of all three representations significantly. See the curves with the same color for comparison.

\subsection{Rotation Estimation without Supervision}

\label{sec:wo_gt}
\noindent\textbf{Self-supervised instance-level rotation estimation from point clouds.}
\label{sec:self-supervised}
For one complete chair instance $Z$, given a complete observation $X$, we estimate its pose $\Rot$. We then use Chamfer distance between $Z$ and $\Rot^{-1}X$ as a self-supervised loss. The network structure and training settings are all the same as Experiment \ref{sec:pc_rotation}, except here we use $\tau=2$. See Appendix \ref{sec:flow loss} for how to find a suitable $\tau$.

The interesting thing here is that vanilla \textbf{9D-Inf} fails while our methods still perform very well. We think that this is because the Chamfer distance loss will greatly enlarge the effect of the noisy part (which is introduced by $\lambda$) in gradient, leading to a very bad performance.

\begin{table}[htbp]

 \caption{\textbf{Self-supervised Instance-Level Rotation Estimation from Point Clouds.} We report mean, median and 3$^\circ$ accuracy of (geodesic) errors after 30K iterations.}
 \resizebox{0.6\columnwidth}{!}{
        \begin{tabular}{lccc}
            Methods& \multicolumn{1}{c}{Mean ($^\circ$)$\downarrow$} & \multicolumn{1}{c}{Med ($^\circ$)$\downarrow$} & \multicolumn{1}{c}{3$^\circ$Acc ($\%$)$\uparrow$} \\
                \midrule  
                Euler       &       131.9 & 139.1 & 0.0   \\
                Axis-Angle   &      4.5 & 3.8 & 34.5   \\
                Quaternion   &      4.3 & 3.5 & 37.5    \\
                6D          &       55.1 & 6.7 & 20.0   \\
                9D          &       1.8 & 1.6 & 88.0  \\
                9D-Inf      &       118.2 & 119.5 &0.0\\
                10D         &       1.6 & 1.5 & 91.0   \\
                \midrule
                RPMG-Quat     &     3.5 & 2.4 & 70.0  \\
                RPMG-6D       &     15.0 & 2.9 & 55.0 \\
                RPMG-9D       &     \textbf{1.3} & \textbf{1.2} & \textbf{97.5}  \\
                RPMG-10D       &     1.5 & 1.4 & 97.0
            \end{tabular}}
        
\end{table}

\vspace{-3mm}
\subsection{Regression on Other Non-Euclidean Manifolds}

\label{sec:other manifolds}
In addition to SO(3), our method can also be applied for regression on other non-Euclidean manifolds as long as the target manifold meets some conditions: 1) the manifold should support Riemannian optimization. 2) the inverse projection $\pi^{-1}$ should be calculable, although it doesn't need to be mathematically complete. Here we show the experiment of \textit{Sphere manifold }$\mathcal{S}^2$. 

\noindent\textbf{Unit vector regression.}
For rotational symmetric categories (e.g., \textit{bottle}), the pose of an object is ambiguous. We'd rather regress a unit vector for each object indicating the \textit{up} direction of it. 
We use the ModelNet-40\cite{wu2015modelnet} \textit{bottle} point cloud dataset. 
The network architecture is the same as in Experiment \ref{sec:pc_rotation} except the dimension of output is 3.

L2-loss-w/-norm computes L2 loss between the normalized predictions and the ground truth. L2-loss-w/o-norm computes L2 loss between the raw predictions and the ground truth, similar to $\lambda=1$ and $\tau=\tau_{gt}$. For MG-3D, PMG-3D and RPMG-3D, We increase $\tau$ from 0.1 to 0.5 since here $\tau_{converge}=0.5$ (please see Appendix \ref{sec:supp3} for the derivation).

The results are shown in Table \ref{tab:unit}. 
MG-3D performs on par with L2-loss-w/o-norm, and PMG-3D leads to a large error since the length vanishing problem similar to Figure \ref{fig:length}. RPMG-3D outperforms all the baselines and variants.

\begin{table}[htbp]
  \centering
  \vspace{-1mm}
  \caption{\textbf{Unit vector estimation from ModelNet bottle point clouds.} We report mean, median, and 1$^\circ$ accuracy of (geodesic) errors after 30K iterations.}
  \vspace{-1mm}
  \begin{minipage}[h]{0.7\columnwidth}
  \resizebox{\columnwidth}{!}{
    \begin{tabular}{lccc}
         
Methods & \multicolumn{1}{l}{Mean ($^\circ$)$\downarrow$} & \multicolumn{1}{l}{Med ($^\circ$)$\downarrow$} & \multicolumn{1}{l}{1$^\circ$Acc ($\%$)$\uparrow$}  \\
\midrule    
    L2 loss w/ norm     & 8.73   & 2.71  & 0.0       \\
    L2 loss w/o norm    & 5.71   &  1.10 & 37.4     \\
    \midrule
    MG-3D ($\lambda$=1)    & 5.37    & 1.20  & 22.2     \\
    PMG-3D ($\lambda$=0)                 & 21.96   &  14.79 & 0.0       \\
    \midrule
    RPMG-3D ($\lambda$=0.01)               & \textbf{4.69}   & \textbf{0.76} & \textbf{72.7}    \\
    \end{tabular}}
    \end{minipage}
  \label{tab:unit}%
\end{table}%
\vspace{-2mm}

\vspace{-3mm}
\section{Conclusion and Future Work}
\label{sec:conclusion}
Our work tackles the problem of designing a gradient layer to facilitate the learning of rotation regression. Our extensive experiments have demonstrated the effectiveness of our method coupled with different rotation representations in diverse tasks dealing with rotation estimation.

The limitation of our methods mainly lies in two fronts: 1) we introduce two new hyperparameters, \textit{i.e.}, $\tau$ and $\lambda$, though our performance is not sensitive to them, as long as they are in a reasonable range; 2) as discussed in Sec \ref{sec:other manifolds}, our method can only be applied to manifolds with certain constraints. We leave how to relax those to future works.


{\small
\bibliographystyle{ieee_fullname}
\bibliography{egbib}
}

\clearpage
\appendix

\section{More on Riemannian Geometry}
\label{sec:supp1}
In this part, we supplement the definitions in Section \ref{sec:prelim} to allow for a slightly more rigorous specification of the exponential map for interested readers.

We denote the union of all tangent spaces as the \emph{tangent bundle}: $\TM = \cup_{\x\in\Man}\TxM$. 
Riemannian metric $\G_\x$ induces a norm $\|\u\|_\x\,,\forall \u\in\TxM$ locally defining the geometry of the manifold and allows for computing the \emph{length} of any curve $\curve : [0,1] \rightarrow \mathcal{M}$, with $\curve(0) = \mathbf{x}$ and $\curve(1) = \mathbf{y}$ as the integral of its speed: $\len(\curve) = \int_{0}^1 \|\dcurve(t)\|_{\curve(t)}dt$. The notion of length leads to a natural notion of distance by taking the infimum over all lengths of such curves, giving the \emph{Riemannian distance} on $\Man$, $d(\x,\y)=\inf_{\curve}\len(\curve)$. The constant speed \emph{length minimizing} curve $\curve$ is called a \emph{geodesic} on $\Man$. 

By the celebrated Picard Lindelöf theorem~\cite{coddington1955theory}, given any $(\x,\v)\in\TM$, there exists a unique \emph{maximal}\footnote{\emph{maximal} refers to the fact that the curve is as long as possible.} geodesic $\geo$ such that $\geo(0)=\x$ and $\dgeo(0)=\v$. Hence, we can define a unique diffeomorphism or \emph{exponential map}, sending $\x$ to the endpoint of the geodesic: $\exp_\x(\v)=\geo(1)$. We will refer to the well-defined, smooth inverse of this map as the \emph{logaritmic map}: $\log_{\x}{\y}\triangleq \exp^{-1}_\x(\v)$. Note that the geodesic is not the only way to move away from $\x$ in the direction of $\v$ on $\Man$. In fact, any continuously differentiable, smooth map $\Rx:\TxM\mapsto \Man$ whose directional derivative along $\v$ is identity, \ie $\mathrm{D} \Rx(\zero)[\v]=\v$ and $\Rx(\zero) = \x$ allows for moving on the manifold in a given direction $\v$. Such $\Rx$, called \emph{retraction}, constitutes the basic building block of any on-manifold optimizer as we use in the main paper. In addition to those we also speak of a \emph{manifold projector} $\pi:\Amb\mapsto\Man$ 
is available for the manifolds we consider in this paper. Note that, most of these definitions directly generalize to matrix manifolds such as Stiefel or Grassmann~\cite{absil2009optimization}.

\section{Projective Manifold Gradient on $SO(3)$}

\subsection{Details of Riemannian Optimization on $SO(3)$}
\label{sec:supp2.1}
\paragraph{Riemannian gradient on $SO(3)$.}
Since we mainly focus on the $\SO$ manifold in this paper, we will further show the specific expression of some related concepts of $\SO$ below.

Firstly, $\SO$ is defined as a matrix subgroup of the general linear group $GL(3)$:
\begin{equation}
\SO=\{\Rot\in\mathbb{R}^{3\times3}:\Rot^\top\Rot=\Id, \det(\Rot)=1\}.
\end{equation}

The tangent space of a rotation matrix in $\SO$ is isomorphic to $\mathbb{R}^3$ making $\SO$ an embedded submanifold of the ambient Eucldiean space $\Amb$. Hence, $\SO$ \emph{inherits} the metric or the inner product of its embedding space, $\Amb$. 

Since $\SO$ is also a Lie group, elements of the tangent space $\bphi^\wedge \in \TiM$ can be uniquely mapped to the manifold $\Man$ through the exponential map:
\begin{equation}\label{eq:expso3}
\exp_{\Id}(\bphi^\wedge) = \Id + \bphi^\wedge + \frac{1}{2!}(\bphi^\wedge)^2+\frac{1}{3!}(\bphi^\wedge)^3+... \quad ,
\end{equation}
where $\Id\in\SO$ is the identity matrix and $^\wedge$ is a skew-symmetric operator $^\wedge:\R^3\to\T_\Id\Man$ as
\begin{equation}
\bphi^\wedge=\begin{pmatrix}
    0 & -\phi_z & \phi_y \\ \phi_z & 0 & -\phi_x \\ -\phi_y & \phi_x & 0 
    \end{pmatrix}
\end{equation}

Due to the nature of the Lie group, we can expand the formula in~\cref{eq:expso3} from the tangent space of the identity, $\TiM$, to $\T_{\Rot}\Man$ by simply multiplying by an $\Rot$:

\begin{equation}
 \exp_\Rot(\bphi^\wedge)
 =\Rot\left(\sum_{n=0}^{\infty} (\frac{1}{n!}(\bphi^\wedge)^n)\right)
 \end{equation}

If the vector $\bphi$ is rewritten in terms of a unit vector $\lie$ and a magnitude $\theta$, the exponential map can further be simplified as 
\begin{equation}
    \Exp_\Rot(\bphi) = \Rot(\Id + \sin\theta~\lie^\wedge+(1-\cos\theta)(\lie^\wedge)^2)
\end{equation}

\noindent which is well known as the Rodrigues formula~\cite{rodrigues1840lois}.
Following \cite{taylor1994minimization}, we have 
\small
\begin{equation}
\left.\frac{\partial}{\partial \phi_x}\Exp_\Rot(\bphi)\right|_{\bphi=\mathbf{0}}=\left.\Rot\left(\cos\theta~\frac{\partial\theta}{\partial\phi_x}\lie^\wedge\right)\right|_{\bphi=\mathbf{0}}=\Rot \x^\wedge
\end{equation}
\normalsize
where $\x=(1,0,0)\in\R^3$. For $\phi_y$ and $\phi_z$, there are the similar expressions of the gradient. Finally we can have
\begin{equation}
\grad~\Loss f(\Rot)
=\left(\left.\frac{\partial f(\Rot)}{\partial \Rot}\frac{\partial}{\partial \bphi} \Exp_\Rot(\bphi)\right|_{\bphi=\mathbf{0}}\right)^\wedge    
\end{equation}

\paragraph{Riemannian gradient descent on $\SO$.}
We are now ready to state the Riemannian optimization in the main paper in terms of the exponential map:

\begin{equation}
    \Rot_{k+1}=\Exp_{\Rot_k}(-\tau_k\nabla \bphi).
\end{equation}

Note that if we consider the most commonly used L2 loss $f(\Rot)=\|\Rot-\Rot_{\gt}\|_F^2$ , where 
\scriptsize
  $$ \Rot=\begin{pmatrix}
    a_1 & b_1 & c_1 \\ a_2 & b_2 & c_2 \\ a_3 & b_3 & c_3 
    \end{pmatrix} \in \SO, \text{\quad} \Rot_{\gt}=\begin{pmatrix}
    x_1 & y_1 & z_1 \\ x_2 & y_2 & z_2 \\ x_3 & y_3 & z_3 
    \end{pmatrix} \in \SO,$$
\normalsize
we can get an analytical expression of $\nabla \bphi=(\nabla \phi_x,\nabla \phi_y,\nabla \phi_z)$ as follows:
\scriptsize
\begin{align}
\label{eq:wgrad}
     \nabla\bphi_x &=\frac{\partial f(\Rot)}{\partial \Rot} *\Rot \x^\wedge \notag\\
        &= 2\left|\left|\begin{pmatrix}
        a_1-x_1 & b_1-y_1 & c_1-z_1 \\ a_2-x_2 & b_2-y_2 & c_2-z_2 \\ a_3-x_3 & b_3-y_3 & c_3-z_3 
        \end{pmatrix}
         \begin{pmatrix}
        0 & c_1 & -b_1 \\ 0 & c_2 & -b_2 \\ 0 & c_3 & -b_3
        \end{pmatrix}\right|\right|_1 \notag\\
        &= 2*\sum^3_{i=1}(b_i*z_i-c_i*y_i)
\end{align}    
\normalsize
Similarly, we have $ \nabla\phi_y=2*\sum^3_{i=1}(c_i*x_i-a_i*z_i)$ and $\nabla\phi_z=2*\sum^3_{i=1}(a_i*y_i-b_i*x_i)$.

\paragraph{$\tau_{converge}$ in ablation study.}
We have mentioned in Section \ref{sec:rpmg} that $\tau$ should be small at the beginning of training and be large when converging. This is because a small $\tau$ can yield $\Rot_g$ closer to $\Rot$ and greatly alleviate the reverse problem at the beginning stage of training discussed in Section \ref{sec:rpmg}. Later in training, a large $\tau$ can help us converge better. The initial $\tau$ will not influence the final results too much, and we just need to choose a reasonable value. But the final $\tau$ matters.

Right before convergence, our ideal choice for the final $\tau$ would be $\tau_{\gt}$.
Given that the value of $\tau_{\gt}$ will change according to the geodesic distance between $\Rot$ and $\Rot_{\gt}$, we instead choose to find a suitable constant value to act like $\tau_{\gt}$ when converging, which we denotes as $\tau_{converge}$.

\begin{lemma}
The final value of $\tau_{converge}$ satisfies:
\small
\begin{equation}\Rot_{\gt}=\lim_{<\Rot,\Rot_{\gt}>\to0} R_\Rot(-\tau_{converge}~\grad~\Loss(f(\Rot)))
\end{equation}
\normalsize
where $<\Rot,\Rot_{\gt}>$ represents the angle between $\Rot$ and $\Rot_{\gt}$.
\end{lemma}
\begin{proof}
Considering the symmetry, without loss of generality, we assume that $\Rot=\Id$, which will simplify the derivation. Based upon the conclusion in~\cref{eq:wgrad}, when we use L2 loss, we have $\nabla \bphi=(2*(z_2-y_3),2*(x_3-z_1),2*(y_1-x_2))$ and $\grad~\Loss f(\Rot)=(\nabla \bphi)^\wedge = 2(\Rot_{\gt}^\top-\Rot_{\gt})$. Taking the manifold logarithm of both sides, we get:
\small
\begin{equation}
\log_\Rot(\Rot_{\gt})=\lim_{<\Rot,\Rot_{\gt}>\to0}-\tau_{converge} ~\grad~\Loss f(\Rot)
\end{equation}
\normalsize
The solution for $\tau_{converge}$ can then be derived as follows:
\small
\begin{align}
\label{eq:tau_safe}
\tau_{converge}&=\lim_{<\Rot,\Rot_{\gt}>\to0}-\frac{\log_\Rot(\Rot_{\gt})}{\grad~\Loss f(\Rot)}
\notag \\ &=\lim_{\theta\to0}-\frac{(\bphi_{\gt})^\wedge}{2(\Rot_{\gt}^\top-\Rot_{\gt})} \notag \\
&=\lim_{\theta\to0}-\frac{(\bphi_{\gt})^\wedge}{2\sin\theta(((\lie_{\gt})^\wedge)^\top-(\lie_{\gt})^\wedge)} \notag \\
&=\lim_{\theta\to0}\frac{\theta}{4\sin\theta} \notag \\
&=\frac{1}{4}
\end{align}
\normalsize
where 
\small 
$(\bphi_{\gt})^\wedge=\log_\Id(\Rot_{\gt})=\theta(\lie_{\gt})^\wedge$, $\theta=<\Id,\Rot_{\gt}>$
\normalsize
\end{proof}

Note that though $\tau_{converge}=\frac{1}{4}$ is only true for the L2 loss, we can solve $\tau_{converge}$ for other frequently used loss formats, \textit{e.g.}, geodesic loss~\cite{peretroukhin_so3_2020}. If we use geodesic loss $\theta^2$, it can be computed that $\tau_{converge}=\frac{1}{2}$. We leave the detailed derivation to the interested readers.

\subsection{Derivations of Inverse Projection}
\label{sec:inverse_proj}
For different rotation representations, we follow the same process to find its inverse projection: we first find the inverse image space $\pi^{-1}(\x_g)$, then project $\x$ to this space resulting in $\x_{gp}$, and finally get our (regularized) projective manifold gradient. 

\paragraph{Quaternion}
We need to solve 
\begin{equation}
    \x_{gp} = \underset{\x_g\in\pi_q^{-1}(\hat{\x}_g)}{\text{argmin}}~\|\x_g-\x\|_2^2,
\end{equation}
where $\x$ is the raw output of our network in \textit{ambient space} $\R^4$, $\hat{\x}_g$ is the next goal in \textit{representation manifold} $\Sphere^3$, and $\x_g$ is the variable to optimize in \textit{ambient space} $\R^4$. Recall $\pi^{-1}_{q}(\hat{\x}_g) = \{\x~|~\x = k\hat{\x}_g, k\in \mathbb{R}~\text{and}~k>0\}$, and we can have
\begin{equation}
\|\x-\x_g\|_2^2=\x^2-2k\x\cdot\hat{\x}_g+k^2\hat{\x}_g^2
\end{equation}
Without considering the condition of $k>0$, We can see when $k=\frac{\x\cdot\hat{\x}_g}{\hat{\x}_g^2}=\x\cdot\hat{\x}_g$ the target formula reaches minimum. Note that when using a small $\tau$, the angle between $\hat{\x}_g$ and $\x$ is always very small, which means the condition of $k=\x\cdot\hat{\x}_g>0$ can be satisfied naturally. For the sake of simplicity and consistency of gradient, we ignore the limitation of $k$ no matter what value $\tau$ takes. Therefore, the inverse projection is $\x_{gp}=(\x\cdot\hat{\x}_g)\hat{\x}_g$.

\paragraph{6D representation}
We need to solve 
\scriptsize
\begin{equation}
    [\u_{gp},\v_{gp}] = \underset{[\u_g,\v_g]\in\pi_{6D}^{-1}([\hat{\u}_g,\hat{\v}_g])}{\text{argmin}}~(\|\u_g-\u\|_2^2+\|\v_g-\v\|_2^2)
\end{equation}
\normalsize
where $[\u,\v]$ is the raw output of network in \textit{ambient space} $\R^6$, $[\hat{\u}_g,\hat{\v}_g]$ is the next goal in \textit{representation manifold} $\mathcal{V}_2(\R^3)$ and $[\u_g,\v_g]$ is the variable to optimize in \textit{ambient space} $\R^6$. Recall $\pi^{-1}_{6D}([\hat{\u}_g,\hat{\v}_g]) = \{[k_1\hat{\u}_g, k_2\hat{\u}_g+k_3\hat{\v}_g]~|~k_1,k_2,k_3 \in \R~\text{and}~k_1, k_3>0\}$. We can see that $\u_g$ and $\v_g$ are independent, and $\u_g$ is similar to the situation of quaternion. So we only need to consider the part of $\v_g$ as below:
\begin{equation}
\|\v-\v_g\|_2^2=\v^2+k_2^2\hat{\u}_g^2+k_3^2\hat{\v}_g^2-2k_2\v\cdot\hat{\u}_g-2k_3\v\cdot\hat{\v}_g
\end{equation}
For the similar reason as quaternion, we ignore the condition of $k_3>0$ and we can see when $k_2=\v\cdot\hat{\u}_g$ and $k_3=\v\cdot\hat{\v}_g$, the target formula reaches minimum. Therefore, the inverse projection is $[\u_{gp},\v_{gp}]=[(\u\cdot\hat{\u}_g)\hat{\u}_g, (\v\cdot\hat{\u}_g)\hat{\u}_g+(\v\cdot\hat{\v}_g)\hat{\v}_g]$

\paragraph{9D representation}
For this representation, obtaining the inverse image $\pi_{9D}^{-1}$ is not so obvious. Recall $\pi_{9D}(\x)=\U\Sigma'\V^\top$, where $\U$ and $\V$ are left and right singular vectors of $\x$ decomposed by SVD expressed as $\x=\U\Sigma \V^\top$, and $\Sigma'=\mathrm{diag}(1,1,\det(\U\V^\top))$.

\begin{lemma}
The inverse image $\pi^{-1}_{9D}(\Rot_g) = \{\mathbf{S}\Rot_g~|~\mathbf{S}=\mathbf{S}^\top\}$ satisfies that $\{\x_g~|~\pi_{9D}(\x_g)=\Rot_g\}\subset \pi^{-1}_{9D}(\Rot_g)$.
\end{lemma}

\begin{proof}
To find a suitable $\pi_{9D}^{-1}$, the most straightforward way is to only change the singular values $\Sigma_g=\mathrm{diag}(\lambda_0,\lambda_1,\lambda_2)$, where $\lambda_0,\lambda_1,\lambda_2$ can be arbitrary scalars, and recompose the $\x_g=\U\Sigma_g \V^\top$. 

However, we argue that this simple method will fail to capture the entire set of $\{\x_g~|~\pi_{9D}(\x_g)=\Rot_g\}$, because different $\U'$ and $\V'$ can yield the same rotation $\Rot_g$. In fact, $\U_g$ can be arbitrary if $\x_g=\U_g\Sigma_g \V_g^\top$ and $\U_g\Sigma_g' \V_g^\top=\Rot_g$. Assuming $\Rot_g$ is known, we can replace $\V_g^\top$ by $\Rot_g$ and express $\x_g$ in a different way: $\x_g=\U_g\Sigma_g\frac{1}{\Sigma_g'}\U_g^{-1}\Rot_g$. Notice that $\U_g\Sigma_g\frac{1}{\Sigma_g'}\U_g^{-1}$ must be a symmetry matrix since $\U_g$ is an orthogonal matrix. Therefore,  $\{\x_g~|~\pi_{9D}(\x_g)=\Rot_g\}\subseteq \pi^{-1}_{9D}(\Rot_g) = \{\mathbf{S}\Rot_g~|~\mathbf{S}=\mathbf{S}^\top\}$.

Note that such $\x_g\in\pi^{-1}_{9D}(\Rot_g)$ can't ensure $\pi_{9D}(\x_g)=\Rot_g$, because in the implementation of SVD, the order and the sign of three singular values are constrained, which is not taken into consideration. Therefore, $\{\x_g~|~\pi_{9D}(\x_g)=\Rot_g\}\neq \pi^{-1}_{9D}(\Rot_g)$.
\end{proof}

Then we need to solve
\begin{equation}
    \x_{gp} = \underset{\x_g\in\pi_{9D}^{-1}(\Rot_g)}{\text{argmin}}~\|\x_g-\x\|_2^2
\end{equation}
where $\x$ is the raw output of our network in \textit{ambient space} $\R^{3\times3}$, $\hat{\x}_g$ is the next goal in \textit{representation manifold} $\SO$, and $\x_g$ is the variable to optimize in \textit{ambient space} $\R^{3\times3}$.
We can further transform the objective function as below: 
\begin{equation}
    \|\x_g-\x\|_2^2 =\|\mathbf{S}\Rot_g-\x\|_2^2
    =\|\mathbf{S}-\x\Rot_g^\top\|_2^2
\end{equation}
Now we can easily find 
when $\mathbf{S}$ equals to the symmetry part of $\x\Rot_g^\top$, the target formula reaches minimum. Therefore, the inverse projection admits a simple form $\x_{gp}=\frac{\x\Rot_g^\top+\Rot_g \x^\top}{2}\Rot_g$. 
\paragraph{10D representation} Recall the \textit{manifold mapping} $\pi_{10D}: \R^{10}~\rightarrow~\Sphere^3, \pi_{10D}(\x)=\underset{\q\in\Sphere^3}{\min}~\q^\top\mathbf{A}(
\x)\q$, in which
\begin{equation}\label{eq:10d}
\vspace{1mm}
    \mathbf{A}(\boldsymbol{\theta})~=~
        \begin{pmatrix}
            \theta_1 & \theta_2 & \theta_3 & \theta_4 \\
            \theta_2 & \theta_5 & \theta_6 & \theta_7 \\
            \theta_3 & \theta_6 & \theta_8 & \theta_9 \\
            \theta_4 & \theta_7 & \theta_9 & \theta_{10} \\
         \end{pmatrix}.
    \vspace{1mm}
\end{equation}

We need to solve
\begin{equation}
    \x_{gp} = \underset{\mathbf{A}(\x_g)\q_g=\lambda\q_g}{\arg\min}~\|\x_g-\x\|_2^2,
\end{equation}
where $\x$ is the raw output of our network in \textit{ambient space} $\R^{10}$, $\q_g$ is the next goal in \textit{representation manifold} $\Sphere^3$, and $\x_g$ is the variable to optimize in \textit{ambient space} $\R^{10}$. Note that $\lambda$ is also a variable to optimize. For the similar reason as before, for the sake of simplicity and consistency of analytical solution, here we also need to relax the constraint that $\lambda$ should be the smallest eigenvalue of $\mathbf{A}(\x_g)$.

To solve Eq. \ref{eq:10d}, we start from rewriting  $\mathbf{A}(\x_g)\q_g=\lambda\q_g$ as 
\begin{equation}
    \M\Delta\x=\lambda\q_g-\mathbf{A}(\x)\q_g,
\end{equation}
where $\Delta\x=\x_g-\x$ and 
\scriptsize
\begin{equation}
    \M=\begin{pmatrix}
        q_1&q_2&q_3&q_4&0&0&0&0&0&0\\
        0&q_1&0&0&q_2&q_3&q_4&0&0&0\\
        0&0&q_1&0&0&q_2&0&q_3&q_4&0\\
        0&0&0&q_1&0&0&q_2&0&q_3&q_4\\
    \end{pmatrix}
\end{equation} 
\normalsize
where $\q_g=(q_1,q_2,q_3,q_4)^\top$.
For simplicity, we denote $\lambda\q_g-\mathbf{A}(\x)\q_g$ as $\bb$. 

Once we have finished the above steps for preparation, we solve $\lambda$ and $\Delta\x$ for the minimal problem by two steps as below. First, we assume $\lambda$ is known and the problem becomes that given $\M$ and $\bb$, we need to find the best $\Delta\x$ to minimize $\|\Delta\x\|^2_2$ with the constraint $\M\Delta\x=\bb$. This is a typical quadratic optimization problem with linear equality constraints, and the analytical solution satisfies

\begin{equation}
    \begin{pmatrix}\label{eq:kkt}
         \mathbf{I} & \M^\top \\
         \M  & \mathbf{0}
    \end{pmatrix}~
    \begin{pmatrix}
         \Delta\x \\
         \mathbf{v}
    \end{pmatrix} = \begin{pmatrix}
         \mathbf{0} \\
        \mathbf{b}
    \end{pmatrix}
\end{equation}

\noindent where $\mathbf{v}$ is a set of Lagrange multipliers which come out of the solution alongside $\Delta\x$, and $\begin{pmatrix}\Id & \M^\top \\ \M & \mathbf{0}\end{pmatrix}$ is called KKT matrix. Since this matrix has full rank almost everywhere, we can multiple the inverse of this KKT matrix in both sides of Eq. \ref{eq:kkt} and lead to the solution of $\Delta\x$ as below:
\begin{equation}
   \begin{pmatrix}
         \Delta\x \\
         \mathbf{v}
    \end{pmatrix} = \begin{pmatrix}
         \mathbf{I} & \M^\top \\
         \M  & \mathbf{0}
    \end{pmatrix}^{-1}~\begin{pmatrix}
         \mathbf{0} \\
        \mathbf{b}
    \end{pmatrix}
\end{equation}
\noindent Recall that $\bb=\lambda\q_g-\mathbf{A}(\x)\q_g$, therefore so far we have had the solution of $\Delta\x$ respect to each $\lambda$:
\begin{equation}\label{eq:kkt_solution}
    \Delta\x=\begin{pmatrix}
         \Delta\x \\
         \mathbf{v}
    \end{pmatrix}_{0:10} = \mathbf{K} (\lambda\q_g-\mathbf{A}(\x)\q_g)=\lambda\mathbf{S}-\mathbf{T}
\end{equation}
in which $\mathbf{K}$ is the upper right part of the inverse of the KKT matrix $\mathbf{K}=\left[\begin{pmatrix}\Id & \M^\top \\ \M & \mathbf{0}\end{pmatrix}^{-1}\right]_{10:14, 0:10}$, $\mathbf{S}=\mathbf{K}\mathbf{q}_g$ and $\mathbf{T}=\mathbf{K}\mathbf{A}(\x)\mathbf{q}_g$.

Next, we need to optimize $\lambda$ to minimize our objective function $\|\Delta\x\|^2_2$. In fact, using the results of Eq. \ref{eq:kkt_solution}, $\|\Delta\x\|^2_2$ becomes a quadratic functions on $\lambda$, thus we can simply get the final analytical solution of $\lambda$ and $\x_{gp}$:
\begin{equation}
    \left\{
    \begin{array}{l}
    \lambda= \frac{(\mathbf{S}^\top\mathbf{T}+\mathbf{T}^\top\mathbf{S})}{2\mathbf{S}^\top\mathbf{S}}\\
    \x_{gp} = \x + \lambda\mathbf{S} - \mathbf{T}
    \end{array}
    \right.
\end{equation}

Another thing worth mentioning here is that in this special case, the \textit{representation manifold} $\Sphere^3$ is no longer a subspace of the \textit{abmient space} $\R^{10}$, which means that we can't directly compute our regularization term $\x_{gp}-\q_g$ because $\x_{gp}\in\R^{10}$ while $\q_g\in\Sphere^3$. However, the length vanishing problem still exists as shown in Figure \ref{fig:length}. Therefore, to compute the regularization term, we need a simple mapping to convert $\q_g$ to an element on $\R^{10}$ with stable length norm. We use the mapping $g:\Sphere^3\rightarrow\R^{10}, g(\q)=\mathbf{A}^{-1}(\mathbf{I}-\q\q^\top)$, which is proposed in \cite{peretroukhin_so3_2020}. They also proved that $\pi(g(\q))=\q$ is always true, which makes $g(\q)$ better than simply normalizing $\x_{gp}$ because the latter one will suffer from the problem of opposite gradient discussed in Section \ref{sec:rpmg}.

\begin{table*}[t]
    \caption{\textbf{Pose estimation from PASCAL3D+ \textit{sofa} images.} Left: a comparison of methods by 10$^\circ$ / 15$^\circ$ / 20$^\circ$ accuracy of (geodesic) errors and median errors after 60k training steps. Middle: median test error at different iterations during training. Right: test error percentiles after training completes. The legend on the right applies to both plots.}

    \begin{minipage}[h]{0.28\columnwidth}
        \resizebox{\columnwidth}{!}{
            \begin{tabular}{lcccc}
                \multirow{2}{*}{Methods}&\multicolumn{3}{c}{ Accuracy$(\%)\uparrow$}&Med$(^\circ)\downarrow$\\
                & 10$^\circ$ & 15$^\circ$  & 20$^\circ$ &  Err \\
                \cmidrule{1-5}  
                Euler       &       60.2& 80.9&90.6&8.3   \\
                Axis-Angle   &      45.0 & 70.9 & 85.1&11.0   \\
                Quaternion   &      34.3&60.8 & 73.5&13.2   \\
                6D          &       50.8& 76.7 & 89.0 & 9.9   \\
                9D          &       52.4 & 79.6 & 90.3&9.2  \\
                9D-Inf      &       70.9 & \textbf{88.0} & 93.5&\textbf{6.7}  \\
                10D         &       50.2&77.0&89.6&9.8   \\
                \midrule
                RPMG-Quat     &     56.6&79.6&90.9&8.9 \\
                RPMG-6D       &     69.6 & 86.1 &92.2&\textbf{6.7}  \\
                RPMG-9D       &     \textbf{72.5} &\textbf{88.0} &
                \textbf{95.8}&\textbf{6.7}  \\
                RPMG-10D       &    69.3 & 87.1 & 93.9 & 7.0\\
            \end{tabular}}
        \end{minipage}
    \begin{minipage}[h]{0.28\columnwidth}
        \includegraphics[width=\columnwidth]{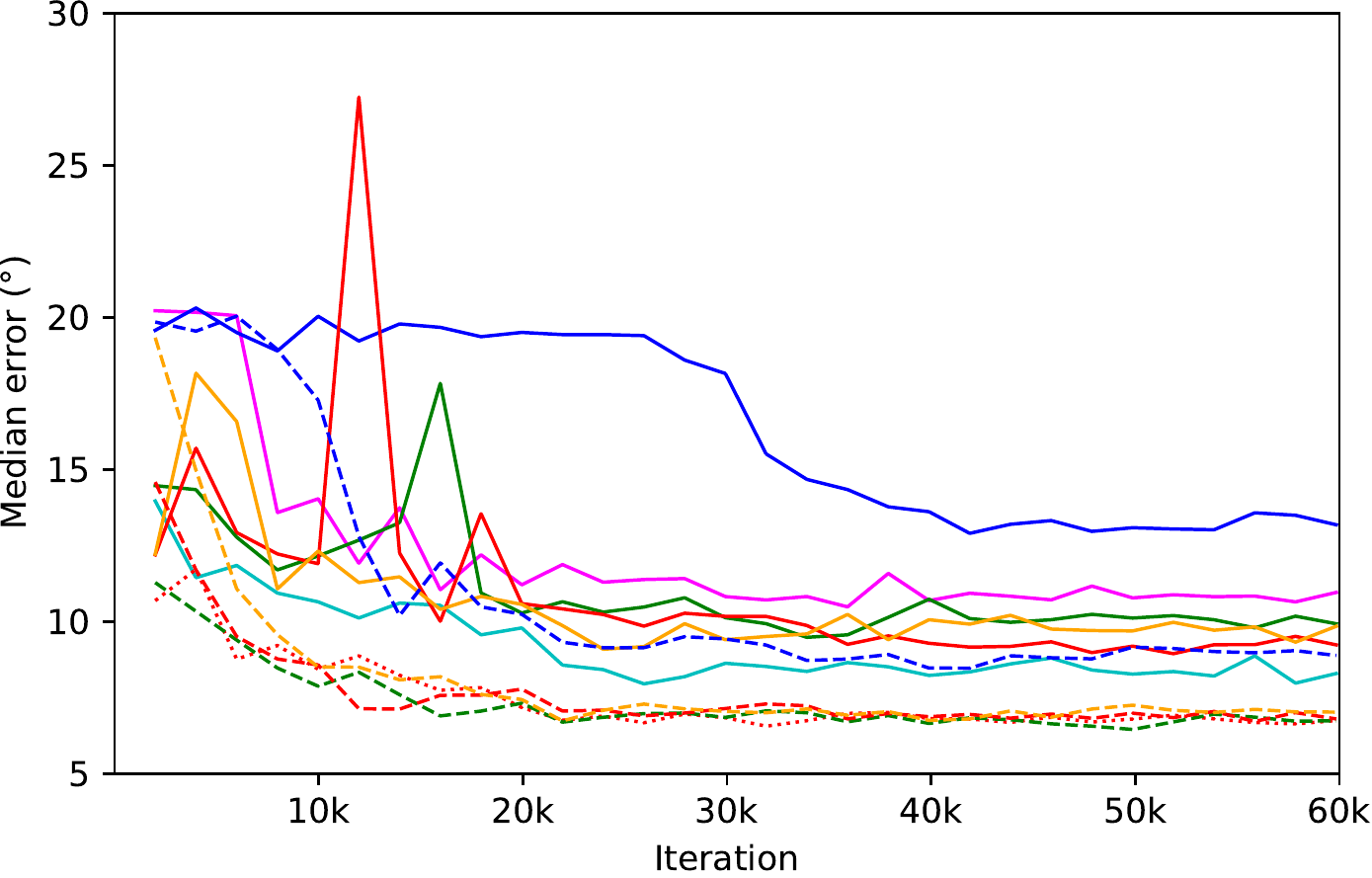}
    \end{minipage}
    \begin{minipage}[h]{0.28\columnwidth}
        \includegraphics[width=\columnwidth]{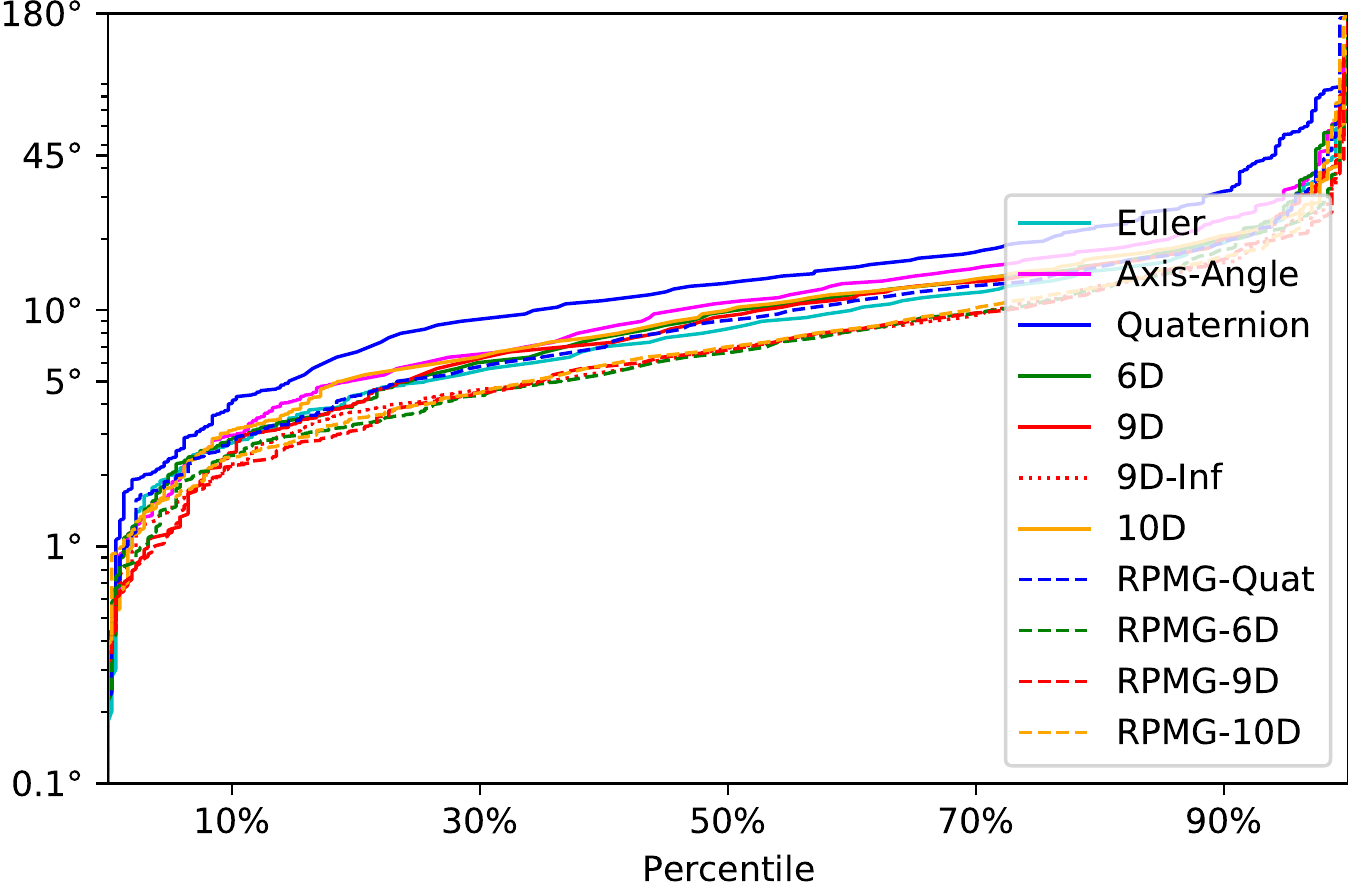}
    \end{minipage}
  \label{tab:pascal3d_sofa}
   \vspace{-1mm}
\end{table*}

\begin{table*}[t]
    \caption{\textbf{Pose estimation from PASCAL3D+ \textit{bicycle} images.} We report the same metrics as Table \ref{tab:pascal3d_sofa}; see the caption there.}
    \begin{minipage}[h]{0.28\columnwidth}
        \resizebox{\columnwidth}{!}{
            \begin{tabular}{lcccc}
              \multirow{2}{*}{Methods} &\multicolumn{3}{c}{ Accuracy$(\%)\uparrow$} &Med$(^\circ)\downarrow$\\
                & 10$^\circ$ & 15$^\circ$  & 20$^\circ$ &  Err \\
                \cmidrule{1-5}  
                Euler       &       28.2& 48.1& 62.7& 15.7   \\
                Axis-Angle   &      5.3 & 8.1 & 10.1&79.7   \\
                Quaternion   &      20.8& 38.8 & 54.6&18.7   \\
                6D          &       21.8& 39.0 & 55.3 & 18.1   \\
                9D          &       20.6 & 37.6 & 56.9&18.0  \\
                9D-Inf      &       38.0 & 53.3 & 69.9&13.4  \\
                10D         &       23.9&42.3&56.7&17.9   \\
                \midrule
                RPMG-Quat     &     32.3&50.0&65.6&15.0 \\
                RPMG-6D       &     35.4 & 57.2 &70.6&13.5  \\
                RPMG-9D       &     36.8&57.4 & \textbf{71.8}&\textbf{12.5}  \\
                RPMG-10D       &    \textbf{40.0} & \textbf{57.7} & 71.3 & 12.9 \\
            \end{tabular}}
        \end{minipage}
    \begin{minipage}[h]{0.28\columnwidth}
        \includegraphics[width=\columnwidth]{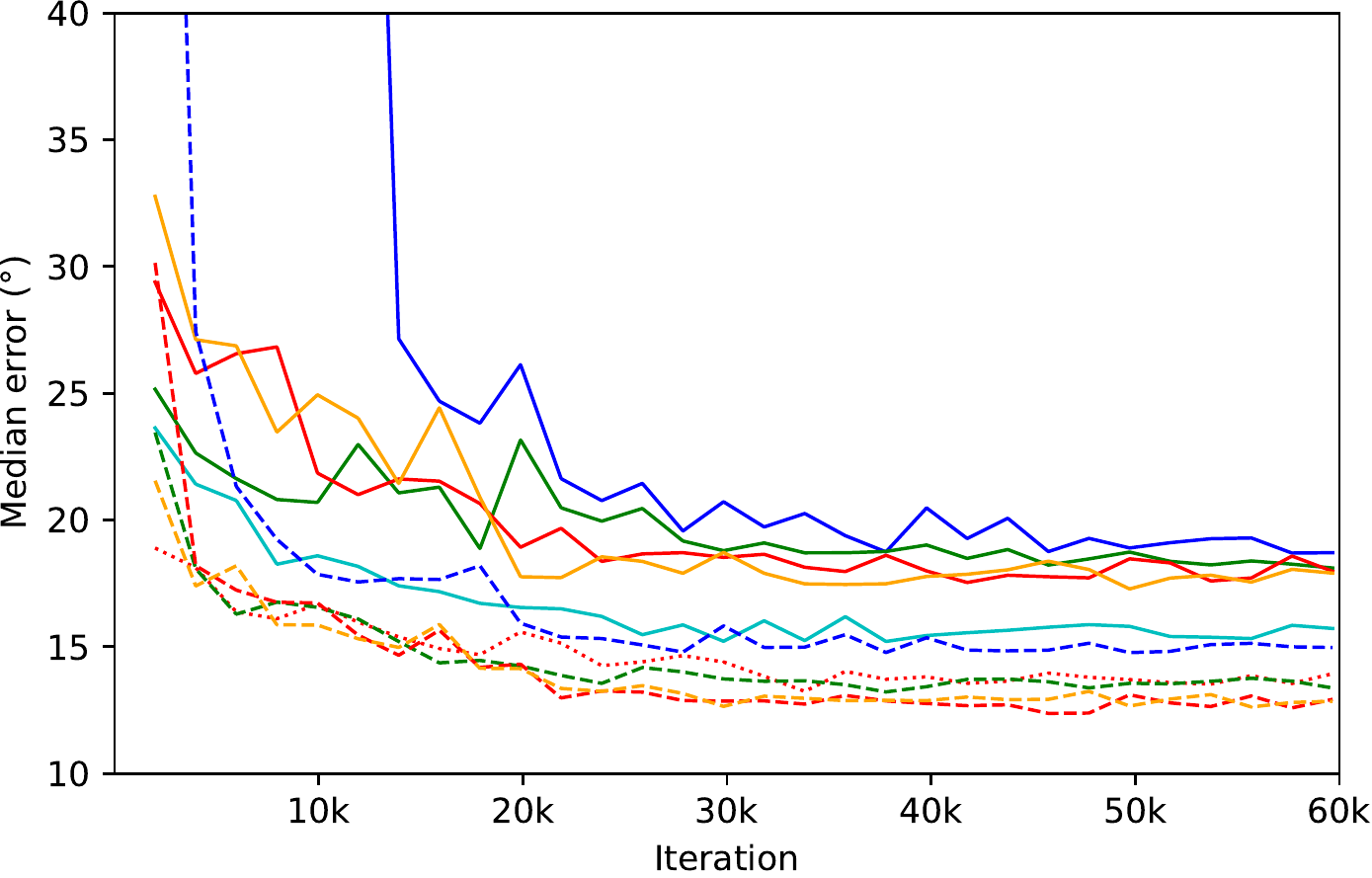}
    \end{minipage}
    \begin{minipage}[h]{0.28\columnwidth}
        \includegraphics[width=\columnwidth]{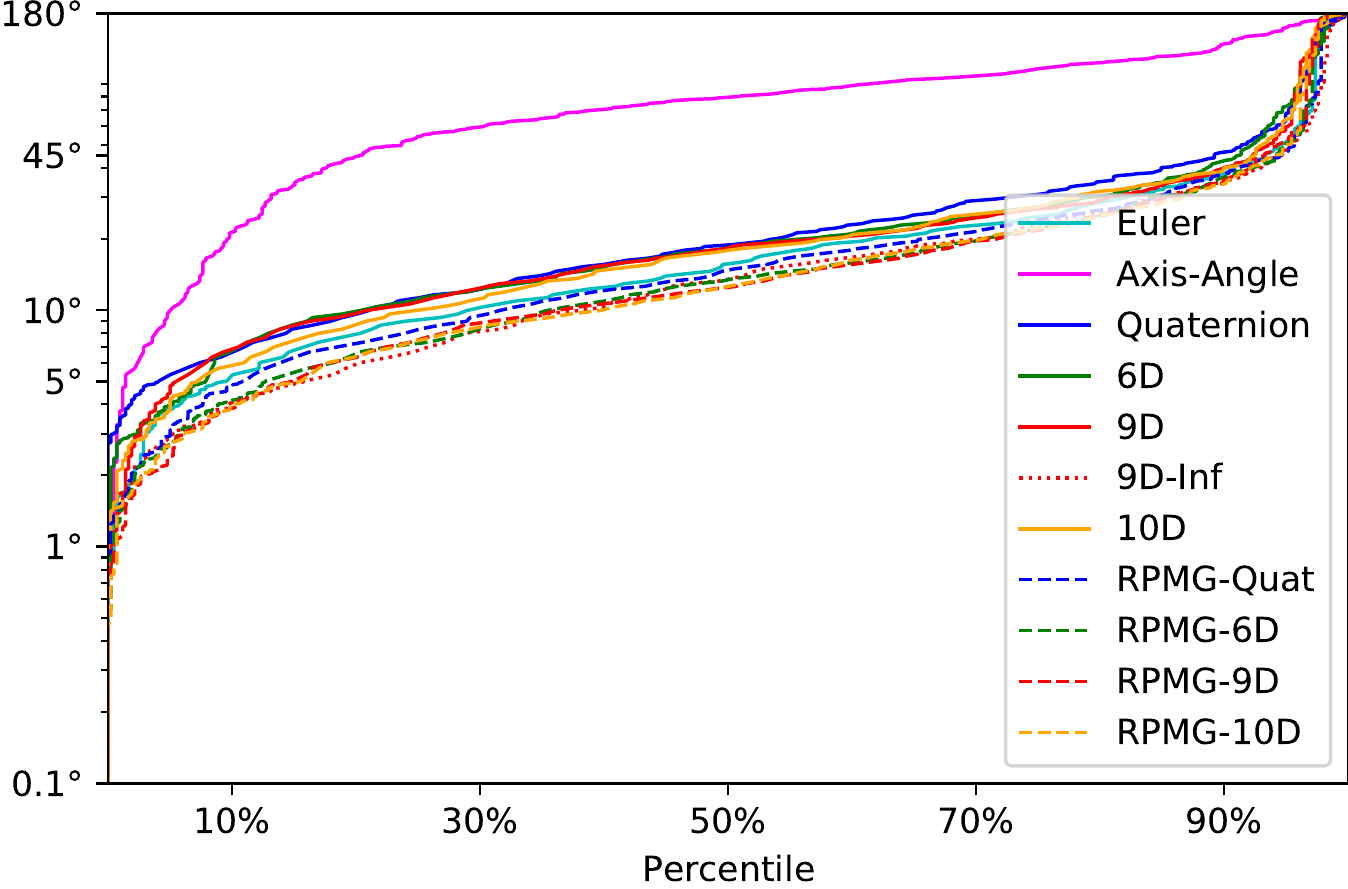}
    \end{minipage}
  \label{tab:pascal3d_bicycle}
  \vspace{-1mm}
\end{table*}

\section{Projective Manifold Gradient on $S^2$}
\label{sec:supp3}
\subsection{Riemannian Optimization on $S^2$}
Our methods can also be applied for the regression of other manifolds. Taking $\Sphere^2$ as an example, which is included in Experiment \ref{sec:other manifolds}, we will show the detail of how our projective manifold gradient layer works in other manifolds.

During forward, The network predicts a raw output $\x\in\R^3$, which is then mapped to $\hat{\x}\in\Sphere^2$ through a \textit{manifold mapping} $\pi(\x)=\x/\|\x\|$. Here we don't define the \textit{rotation mapping} and \textit{representation mapping}, and we directly compute the loss function on \textit{representation manifold} $\Sphere^2$.

During backward, to apply a Riemannian optimization, we first need to know some basic concepts of $\Sphere^2$. The tangent space of an arbitrary element $\hat{\x}\in\Sphere^2$ is $\T_{\hat{\x}}\Man$, which is a plane. And we can map a geodesic path $\v\in\T_{\hat{\x}}\Man$ to an element on the manifold $\Sphere^2$ through $\exp_{\hat{\x}}(\v)=\cos(\|\v\|)\hat{\x}+\sin(\|\v\|)\frac{\v}{\|\v\|}$, where $\|.\|$ means the ordinal Frobenius norm. 

For the definition of the mapping $^\wedge$, which connects Euclidean space $\R^2$ and the tangent space $\T_{\hat{\x}}\Man$, we need to first define two orthogonal axes $\hat{\mathbf{c}}_1$, $\hat{\mathbf{c}}_2$ in the tangent plane. Note that the choice of $\hat{\mathbf{c}}_1$ and $\hat{\mathbf{c}}_2$ won't influence the final result, which will be shown soon after. To simplify the derivation, we can assume ground truth unit vector $\hat{\x}_{gt}$ is known and choose $\hat{\mathbf{c}}_1=\frac{\Log_{\hat{\x}}(\hat{\x}_{gt})}{\|\Log_{\hat{\x}}(\hat{\x}_{gt})\|}=\frac{\hat{\x}_{gt}-(\hat{\x}_{gt}\cdot\hat{\x})}{\|\hat{\x}_{gt}-(\hat{\x}_{gt}\cdot\hat{\x})\|}$ and $\hat{\mathbf{c}}_2=\hat{\x}\times\hat{\mathbf{c}}_1$. Then we can say $\bphi^\wedge=\phi_1\hat{\mathbf{c}}_1+\phi_2\hat{\mathbf{c}}_2$, where $\bphi=(\phi_1, \phi_2)\in\R^2$. The gradient of exponential mapping with respect to $\bphi$ is
\small
\begin{align}
    &\left.\frac{\partial}{\partial \phi_1}\Exp_{\hat{\x}}(\bphi)\right|_{\bphi=\mathbf{0}} \notag \\
    =&\left.\frac{\partial}{\partial \phi_1}(\cos(\|\phi_1\hat{\mathbf{c}}_1\|)\hat{\x}+ 
    \sin(\|\phi_1\hat{\mathbf{c}}_1\|)\frac{\phi_1\hat{\mathbf{c}}_1}{\|\phi_1\hat{\mathbf{c}}_1\|})\right|_{\bphi=\mathbf{0}} \notag \\ 
    =&~\hat{\mathbf{c}}_1
\end{align}
\normalsize
Similarly, we have $\left.\frac{\partial}{\partial \phi_2}\Exp_{\hat{\x}}(\bphi)\right|_{\bphi=\mathbf{0}}=\hat{\mathbf{c}}_2$.

When using L2 loss, we can have 
\small
\vspace{-1mm}
\begin{align}
\grad~\Loss f(\hat{\x})&=(\nabla f(\hat{\x}))^\wedge=(\nabla \bphi)^\wedge\notag\\
&=\left(\left.\frac{\partial f(\hat{\x})}{\partial \hat{\x}}\frac{\partial}{\partial \bphi} \Exp_{\hat{\x}}(\bphi)\right|_{\bphi=\mathbf{0}}\right)^\wedge \notag\\
&=((2(\hat{\x}-\hat{\x}_{gt})\hat{\mathbf{c}}_1,2(\hat{\x}-\hat{\x}_{gt})\hat{\mathbf{c}}_2))^\wedge \notag\\
&=2((\hat{\x}\cdot\hat{\x}_{gt})\hat{\x}-\hat{\x}_{gt})
\end{align}
\normalsize
Note that this expression doesn't depend on the choice of $\hat{\mathbf{c}}_1$ and $\hat{\mathbf{c}}_2$.

Similar to Eq \ref{eq:tau_safe}, we can also solve a $\tau_{converge}$
\scriptsize
\begin{equation}
\vspace{-1mm}
\tau_{converge}=\lim_{<\hat{\x},\hat{\x}_{\gt}>\to0}-\frac{\Log_{\hat{\x}}(\hat{\x}_{\gt})}{\grad~\Loss f(\hat{\x})}=\lim_{\theta\to0}\frac{\theta\hat{\mathbf{c}}_1}{2\sin\theta\hat{\mathbf{c}}_1} =\frac{1}{2}
\end{equation}
\normalsize
where $\theta=<\hat{\x}, \hat{\x}_{gt}>$. Note that in Experiment 5.4, we change the schedule of $\tau$ according to this conclusion. We increase $\tau$ from $0.1$ to $0.5$ by uniform steps.

\subsection{Inverse Projection}

Similar to quaternion, we can have $\x_{gp}=(\x\cdot\hat{\x}_g)\hat{\x}_g$. For the detail of derivation, see Section \ref{sec:inverse_proj}.

\section{Computational Cost}
\label{sec:cost}
Our method does not alter the forward pass and thus incurs no cost at test time. For backward pass at training, we observe that, before and after inserting RMPG layers, the backward time for quaternion / 6D / 9D / 10D representations, averaged among 1K iterations on a GeForce RTX 3090, changes from 4.39 / 4.48 / 4.48 / 4.53  to 4.45 / 4.43 / 4.49 / 4.63 (unit: $10^{-2}$ s), and the memory cost changes from 11449 / 11415 / 10781 / 11363 to 11457 / 11459 / 11545 / 11447 (unit: MiB). Note that the runtime is almost keep the same, as Riemannian optimization only performs an additional projection and we derive and always use analytical solutions in representation mapping, inversion and projection steps. RPMG also has a very marginal cost on the memory, as it does not introduce any weights but only a few intermediate variables.

\section{More Experiments}

\label{sec:more_exps}

\subsection{Pascal3D+}
\label{pascal3d+}
Pascal3D+ \cite{xiang2014pascal3d} is a standard benchmark for object pose estimation from real images. 
We follow the same setting as in \cite{levinson2020analysis} to estimate object poses from single images. For training we discard occluded or truncated objects and augment with rendered images from \cite{renderforcnn15}. In the Table \ref{tab:pascal3d_sofa} and Table \ref{tab:pascal3d_bicycle}, we report our results on \textit{sofa} and \textit{bicycle} categories.
We use the same batch size as in \cite{levinson2020analysis}. As for the learning rate, we use the same strategy as in Experiment \ref{sec:image}. See the discussion in Section \ref{sec:img_lr}.

It can be seen that our method leads to consistent improvements to quaternion, 6D, 9D and 10D representations on both \textit{sofa} and \textit{bicycle} classes. One may be curious about why our method can only outperform 9D-inf for a margin. We think that this is because this dataset is quite challenging. The number of annotated real image for training is only around 200 for each category. Though there are a lot of synthetic images generated from \cite{renderforcnn15} for training, these images suffer from sim-to-real domain gap. Therefore, we argue that the bottleneck here is not in optimization, which makes the gains from less noise in gradient smaller(Note that 9D-inf is just a special case of our methods with $\lambda=1$ and  $\tau=\tau_{gt}$). But compared to vanilla 4D/6D/9D/10D representation, our methods can still bring a great improvement.

\subsection{Using Flow Loss for Rotation Estimation from Point Clouds.}
\label{sec:flow loss}
Apart from the most widely used L2 loss, our method can also be applied to the loss of other forms, e.g. flow loss.

We mainly follow the setting in Experiment \ref{sec:pc_rotation} with \textit{airplane} point clouds dataset and the only difference is that we use flow loss $\|\mathbf{R}X-\mathbf{R}_{gt}X\|_F^2$ here, where $X$ is the complete point clouds. 

Since the format of loss is changed, the previous schedule of $\tau$ is not suitable anymore, and we have to change the value of $\tau$ accordingly. Our selection skill is to first choose a $\tau$ as we like and visualize the mean geodesic distance between predicted $\Rot$ and $\Rot_g$ during training. Then we can roughly adjust $\tau$ to make the geodesic distance looked reasonable. For this experiment, we use $\tau=50$ and $\lambda=0.01$. 
In Table \ref{tab:flow-loss}, we show our methods again outperform vanilla methods as well as \textbf{9D-inf}.

\begin{table}[htbp]
    \centering
    \caption{\textbf{Flow Loss for Rotation Estimation from Point Clouds.}  All models are trained for 30K iterations. }
     \resizebox{0.8\columnwidth}{!}{
    \begin{tabular}{lccc}
               
               Methods & \multicolumn{1}{c}{Mean ($^\circ$)} & \multicolumn{1}{c}{Med ($^\circ$)} & \multicolumn{1}{c}{5$^\circ$Acc ($\%$)} \\
                \cmidrule{1-4}  
                Euler       &       12.14 & 6.91 & 33.6 \\
                Axis-Angle   &      35.49 & 20.80 & 4.7  \\
                Quaternion   &      11.54 & 7.67 & 29.8 \\
                6D          &       14.13 & 9.41 & 23.4   \\
                9D          &       11.44 & 8.01 & 23.8  \\
                9D-Inf      &       4.07 & 3.28 & 76.7 \\
                10D         &       9.28 & 7.05 & 32.6  \\
                \midrule
                RPMG-Quat     &     4.86 & 3.25 & 75.8   \\
                RPMG-6D       &     \textbf{2.71} & \textbf{2.04} & \textbf{92.1} \\
                RPMG-9D       &     3.75 & 2.10 & 91.1 \\
                RPMG-10D & 3.30 & 2.70 & 86.8 \\
            \end{tabular}}
    \label{tab:flow-loss}
     \vspace{-3mm}
\end{table}

\begin{table*}[htbp]
    \centering
    \caption{\textbf{Camera relocalization on Cambridge Landscape dataset.} We report the \textit{median} error of translation and rotation of the best checkpoint, which is chosen by minimizing the median of rotation. We only care about the rotation error here.}
     \begin{tabular}{l|cc|cc|cc|cc|cc}
    \multirow{2}{*}{Methods}& \multicolumn{2}{c|}{King's College} & \multicolumn{2}{c|}{Old Hospital}& \multicolumn{2}{c|}{Shop Facade}&
            \multicolumn{2}{c|}{St Mary's Church} & \multicolumn{2}{c}{Average} \\
    \cmidrule{2-11}
    & T($m$)& R($^\circ$) & T($m$)& R($^\circ$) & T($m$)& R($^\circ$) & T($m$)& R($^\circ$)& T($m$)& R($^\circ$) \\
    \midrule
    Euler&1.16&2.85&2.54&2.95&1.25&6.48&1.98&6.97&1.73&4.81\\
    Axis-Angle&1.12&2.63&2.41&3.38&\textbf{0.84}&5.05&2.16&7.58&1.63&4.66\\
    Quaternion&\textbf{0.98}&2.50&2.39&3.44&1.06&6.01&2.59&8.81&1.76&5.19\\
    6D&1.10&2.56&2.21&3.43&1.01&5.43&\textbf{1.73}&5.82&1.51&4.31\\
    9D&1.14&3.03&2.11&3.50&0.88&6.39&1.95&5.95&1.52&4.72\\
    9D-Inf&\textbf{0.98}&2.32&\textbf{1.89}&3.32&1.15&6.36&1.96&6.25&\textbf{1.50}&4.56\\
    10D&1.54&2.62&2.32&3.39&1.20&5.76&1.85&6.69&1.73&4.62\\
    \midrule
    RPMG-Quat&1.04&1.91&2.42&2.72&0.98&4.28&1.82&4.89&1.57&3.45\\
    RPMG-6D&1.55&\textbf{1.70}&2.62&3.09&0.95&5.01&2.44&5.18&1.89&3.75\\
    RPMG-9D&1.57&1.82&4.37&3.12&0.93&4.17&1.92&\textbf{4.69}&2.20&3.45\\
    RPMG-10D & 1.30 & 1.74 & 3.21 & \textbf{2.59} & 1.10 & \textbf{3.47} &2.20&5.09 & 1.95 & \textbf{3.22} \\
\end{tabular}
    \label{tab:reloc}
\end{table*}

\subsection{Camera Relocalization}
\label{sec:reloc}
The task of camera relocalization is to estimate a 6 Degree-of-Freedom camera pose (rotation and translation) from visual observations, which is a fundamental component of many computer vision and robotic applications.
In this experiment, we use all the settings (data, network, training strategy, hyperparameters, etc.) of PoseLSTM \cite{PoseLSTM17} except that we modify the rotation representations and the gradient layers. We report the results on the outdoor Cambridge Landscape dataset \cite{posenet} in Table \ref{tab:reloc}.

Notice that our RPMG layer performs the best on the rotation regression task, but not on the translation regression. We believe this results from a loss imbalance. We does not change the weights of the rotation loss and translation loss, otherwise it leads to an unfair comparison with existing results. We only care about the rotation error here.

\section{More Implementation Details}
\label{sec:imp_detail}
\subsection{Experiment \ref{sec:pc_rotation} \& \ref{sec:self-supervised} \& \ref{sec:other manifolds}}
\textbf{Data}
We generate the data from ModelNet dataset \cite{wu2015modelnet} by sampling 1024 points on the mesh surface, following the same generation method as in \cite{zhou2019continuity}. We uniformly sample M rotations for each data point and set them as the ground truth. We apply the sampled rotations on the canonical point clouds to obtain the input data. 

\textbf{Network Architecture}
We use a PointNet++ MSG \cite{qi2017pointnetplusplus} backbone as our feature extractor. Our network takes input a point cloud with a resolution of 1024. It them performs three set abstractions to lower the resolution to 512, 128, and finally 1, resulting in a global feature of dimensionality 1024. The feature is finally pushed through a three-layer MLP $[1024, 512, N]$ to regress rotation, where $N$ is the dimension of the rotation representation. 

\textbf{Training details} The learning rate is set to 1e-3 and decayed by 0.7 every 3k iterations. The batch size is 20. For each experiment, we train the network on one NVIDIA TITAN Xp GPU for 30k iterations.

\subsection{Experiment \ref{sec:image}}
\label{sec:img_lr}
Most of the training settings and strategies are all the same as \cite{levinson2020analysis} except learning rate. We find setting initial learning rate $lr=1e-3$ and decaying to $1e-5$ can perform much better than using $lr=1e-5$ as in \cite{levinson2020analysis}, which accounts for the inconsistency of the results of those baseline methods compared to \cite{levinson2020analysis}. We believe that the methods should be compared under hyperparameters as optimal as possible. Thus, we stick to our $lr$ schedule. 

\begin{table}[htbp]
    \caption{\textbf{Test error percentiles for Experiment \ref{sec:pc_rotation} \& \ref{sec:image}} Left: test error percentiles of \textit{airplane} for Experiment \ref{sec:pc_rotation} after training completes. Right: test error percentiles of \textit{chair} for Experiment 5.2 after training completes.}

    \begin{minipage}[h]{0.48\columnwidth}
        \includegraphics[width=\columnwidth]{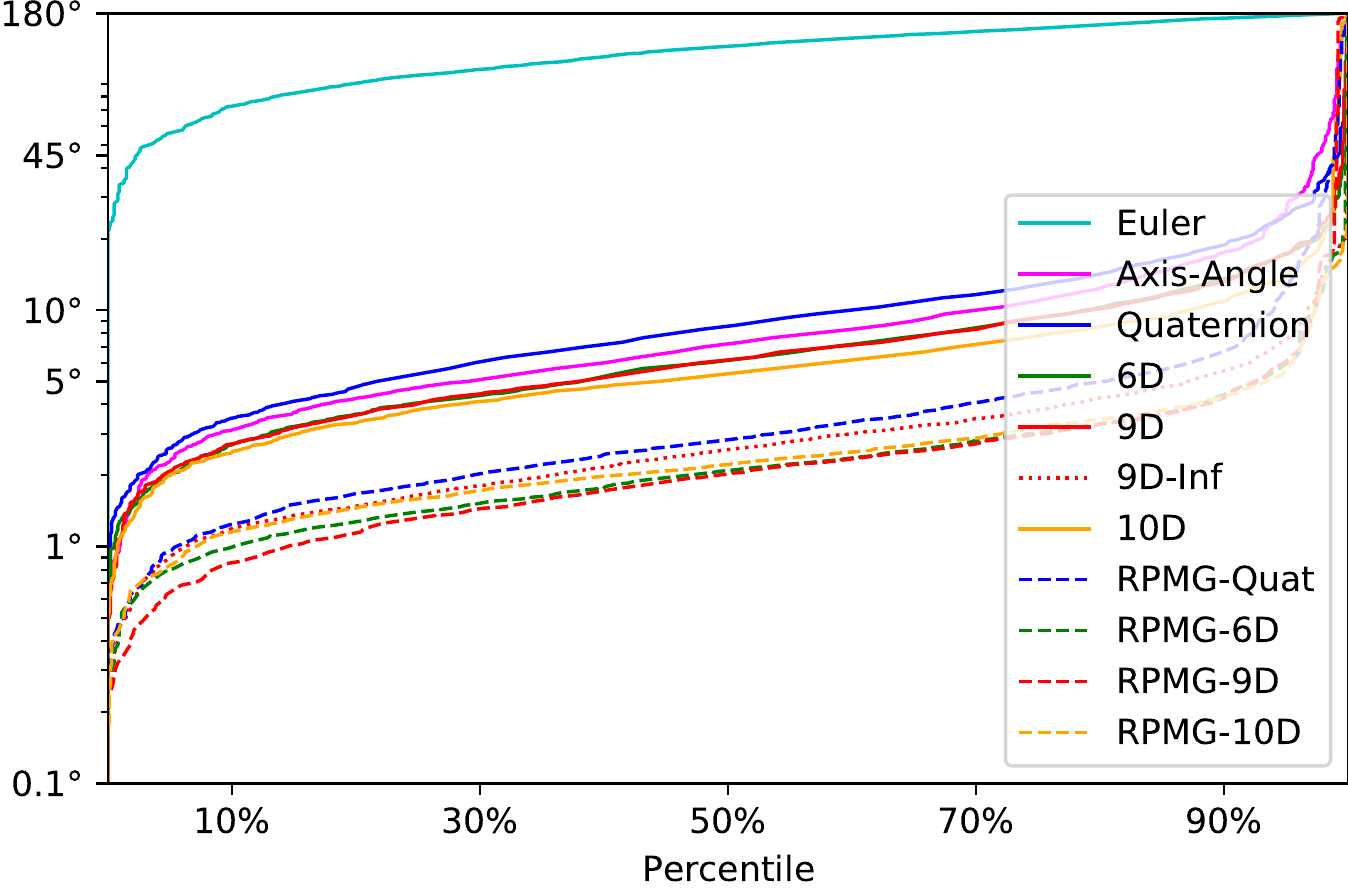}
    \end{minipage}
    \begin{minipage}[h]{0.48\columnwidth}
        \includegraphics[width=\columnwidth]{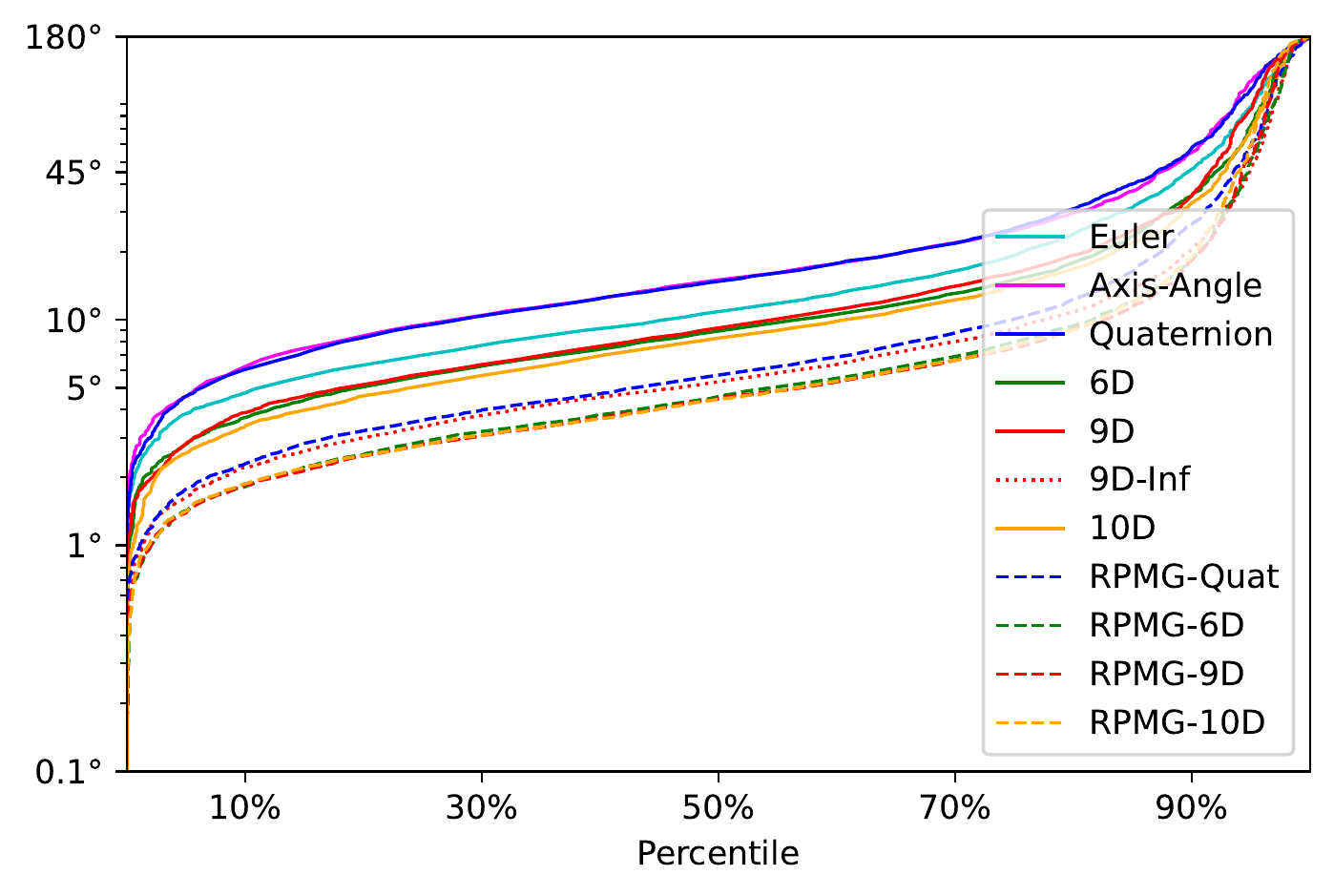}
    \end{minipage}
  \label{tab:add_fig}
   \vspace{-1mm}
\end{table}

\section{Addition on Rotation Representations}
\label{sec:q_r_trans}

\paragraph{Standard mapping between rotation matrix and unit quaternion}

The \emph{rotation mapping} $\phi: \q\mapsto \Rot$ algebraically manipulates a unit quaternion $\q$ into a rotation matrix: 
\scriptsize
\begin{equation}\phi(\q)=\begin{pmatrix}
        2(q_0^2+q_1^2)-1&2(q_1q_2-q_0q_3)&2(q_1q_3+q_0q_2)\\
        2(q_1q_2+q_0q_3)&2(q_0^2+q_2^2)-1&2(q_2q_3-q_0q_1)\\
        2(q_1q_3-q_0q_2)&2(q_2q_3+q_0q_1)&2(q_0^2+q_3^2)-1\\
\end{pmatrix}\end{equation}
\normalsize
where $\q=(q_0,q_1,q_2,q_3)\in\Sphere^3$. 

In the reverse direction, the \emph{representation mapping} $\psi(\Rot)$ can be expressed as:
\small 
\begin{equation}\left\{
\begin{array}{l}
      q_0=\sqrt{1+R_{00}+R_{11}+R_{22}}/2 \\
      q_1=(R{21}-R_{12})/(4*q_0)\\
      q_2=(R_{02}-R_{20})/(4*q_0)\\
      q_3=(R_{10}-R_{01})/(4*q_0)\\
\end{array}
\right.\end{equation}
\normalsize
Note that $\q=(q_0,q_1,q_2,q_3)$ and $-\q=(-q_0,-q_1,-q_2,-q_3)$ both are the valid quaternions parameterizing the same $\Rot$.

\end{document}